\documentclass[twoside]{article}

\usepackage[accepted]{aistats2024}

\usepackage[round]{natbib}

\usepackage[ruled,linesnumbered]{algorithm2e}
\usepackage[final, babel]{microtype}
\usepackage[noend]{algpseudocode}
\usepackage{amsmath,amsthm,amsfonts,amssymb}
\usepackage{hyperref}
\usepackage{color}
\usepackage{mathrsfs}
\usepackage{enumitem}
\usepackage{multirow}
\usepackage{booktabs}
\usepackage{makecell}
\usepackage{graphicx}
\usepackage{subfigure}
\usepackage{caption}
\usepackage{thmtools}
\usepackage{thm-restate}
\usepackage{hhline}
\usepackage{cite}
\usepackage[table]{xcolor}
\usepackage{rotating} 
\usepackage{bm}
\usepackage{float}

\renewcommand{\epsilon}{\varepsilon}

\newcommand{\trans}{^{\top}}

\newcommand{\cA}{\mathcal{A}}
\newcommand{\cB}{\mathcal{B}}

\newcommand{\cE}{\mathcal{E}}
\newcommand{\cF}{\mathcal{F}}
\newcommand{\cG}{\mathcal{G}}
\newcommand{\cH}{\mathcal{H}}

\newcommand{\cM}{\mathcal{M}}
\newcommand{\cN}{\mathcal{N}}

\newcommand{\cP}{\mathcal{P}}

\newcommand{\cS}{\mathcal{S}}

\newcommand{\cX}{\mathcal{X}}
\newcommand{\cY}{\mathcal{Y}}
\newcommand{\cZ}{\mathcal{Z}}

\newcommand{\EE}{\mathbb{E}}
\newcommand{\mH}{\mathbb{H}}

\newcommand{\matD}{\textbf{\text{D}}}
\newcommand{\matDTV}{\mathbb{TV}}

\newcommand{\bD}{\bar{D}}
\newcommand{\ts}{\tilde{s}}
\newcommand{\ta}{\tilde{a}}

\newcommand{\tz}{\tilde{z}}

\newcommand{\tpsi}{{\tilde{\psi}}}

\newcommand{\hM}{{\hat{M}}}
\newcommand{\hcM}{{\hat{\cM}}}
\newcommand{\ph}{{h'}}

\newcommand{\tM}{{\tilde{M}}}
\newcommand{\tT}{{\tilde{T}}}

\newcommand{\tpi}{{\tilde{\pi}}}

\newcommand{\mP}{{\mathbb{P}}}
\newcommand{\mR}{{\mathbb{R}}}

\newcommand{\mN}{{\mathbb{N}}}

\newcommand\numberthis{\addtocounter{equation}{1}\tag{\theequation}}

\let\hat\widehat
\let\tilde\widetilde

\newtheorem{theorem}{Theorem}[section]
\newtheorem{lemma}[theorem]{Lemma}
\newtheorem{corollary}[theorem]{Corollary}

\newtheorem{proposition}[theorem]{Proposition}

\theoremstyle{definition}
\newtheorem{definition}[theorem]{Definition}

\newtheorem{assumption}{Assumption}

\newtheorem{remark}[definition]{Remark}

\newcommand{\vectorize}{\textbf{vec}}

\newcount\Comments  %
\newcommand{\kibitz}[2]{\ifnum\Comments=1\textcolor{#1}{#2}\fi}

\newcommand{\blue}[1]{{\color{blue} #1}}
\newcommand{\red}[1]{{\color{red} #1}}

\newcommand{\diag}{{\rm Diag}}

\DeclareMathOperator*{\argmax}{arg\,max}
\DeclareMathOperator*{\argmin}{arg\,min}

\newcommand{\dimE}{\text{dimE}}

\DeclareSymbolFont{largesymbolsA}{U}{txexa}{m}{n}
\DeclareMathSymbol{\varprod}{\mathop}{largesymbolsA}{16}

\DeclareFontFamily{U}{mathx}{\hyphenchar\font45}
\DeclareFontShape{U}{mathx}{m}{n}{
      <5> <6> <7> <8> <9> <10>
      <10.95> <12> <14.4> <17.28> <20.74> <24.88>
      mathx10
      }{}
\DeclareSymbolFont{mathx}{U}{mathx}{m}{n}
\DeclareMathSymbol{\bigtimes}{1}{mathx}{"91}

\newcommand{\Tr}{\text{Tr}}

\newcommand{\TV}{\mathbb{TV}}

\newcommand{\textR}{r}

\newcommand{\vecmu}{{\bm{\mu}}}

\newcommand{\tldh}{{\tilde{h}}}

\newcommand{\MFC}{{\text{MFC}}}
\newcommand{\MFG}{{\text{MFG}}}
\newcommand{\MLE}{{\text{MLE}}}

\newcommand{\NE}{\text{NE}}
\newcommand{\eff}{\text{eff}}

\newcommand{\bcS}{\bar{\cS}}

\newcommand{\Opt}{\text{Opt}}

\newcommand{\SumInt}{\sum}

\begin{document}

\runningtitle{On the Statistical Efficiency of MFRL with General Function Approximation}

\twocolumn[

\aistatstitle{On the Statistical Efficiency of Mean-Field Reinforcement Learning with General Function Approximation}

\aistatsauthor{%
  Jiawei Huang \quad \quad Batuhan Yardim \quad \quad Niao He
}
\aistatsaddress{
  Department of Computer Science\\
  ETH Zurich \\
  \texttt{\{jiawei.huang, alibatuhan.yardim, niao.he\}@inf.ethz.ch} \\
} 

]

\begin{abstract}
    In this paper, we study the fundamental statistical efficiency of Reinforcement Learning in Mean-Field Control (MFC) and Mean-Field Game (MFG) with general model-based function approximation. We introduce a new concept called Mean-Field Model-Based Eluder Dimension (MF-MBED), which characterizes the inherent complexity of mean-field model classes. We show that a rich family of Mean-Field RL problems exhibits low MF-MBED. Additionally, we propose algorithms based on maximal likelihood estimation, which can return an $\epsilon$-optimal policy for MFC or an $\epsilon$-Nash Equilibrium policy for MFG. The overall sample complexity depends only polynomially on MF-MBED, which is potentially much lower than the size of state-action space. Compared with previous works, our results only require the minimal assumptions including realizability and Lipschitz continuity.
\end{abstract}

\allowdisplaybreaks

\section{INTRODUCTION}
Multi-Agent Reinforcement Learning (MARL) addresses how multiple autonomous agents cooperate or compete with each other in a shared environment, and it is widely applied for practical problems in many areas, including autonomous driving \citep{shalev2016safe}, finance \citep{lee2007multiagent}, and robotics control \citep{ismail2018survey}. Although MARL has attracted increasing attention in recent RL research \citep{zhang2021multi}, when the number of agents is in the hundreds or thousands, solving MARL becomes a challenge.
However, in scenarios where agents exhibit symmetry, like humans in crowds or individual cars in the traffic flow, mean-field theory can be employed to approximate the system dynamics, which results in the Mean-Field RL (MFRL) setting.
In MFRL, the interaction within large populations is reflected by the dependence on the state density of the transition and reward functions of individual agents.
The mean-field model has achieved success in various domains, including economics \citep{cousin2011mean, angiuli2021reinforcement}, finance \citep{cardaliaguet2018mean}, industrial engineering \citep{de2019mean}, etc.

Depending on the objectives, MFRL can be divided into two categories: Mean-Field Control (MFC) and Mean-Field Game (MFG) \citep{lasry2007mean, huang2006large,bensoussan2013mean}.
MFC, similar to the single-agent RL, aims to find a policy maximizing the expected return, while MFG focuses on identifying the Nash Equilibrium (NE) policy, where no agent tends to deviate.
Compared with single-agent RL, MFRL is much more challenging because it requires exploration in the joint space of state, action, and state density, especially given that the density belongs to an infinite and continuous space.

The sample efficiency (referring to the number of samples needed to explore the environment and achieve objectives) in both tabular and more broadly function approximation settings has been extensively examined as one of the fundamental questions in reinforcement learning, particularly in single-agent scenarios.
See for example existing work in single-agent \citep{auer2008near, azar2017minimax, jin2018q, russo2013eluder, jin2021bellman, du2021bilinear, xie2022role, foster2021statistical} and general multi-agent settings \citep{jin2021v,bai2020near,xie2020learning, huang2021towards,wang2023breaking,cui2023breaking}.
However, the understanding of fundamental sample efficiency in existing MFRL literature is still limited in the following two aspects.
\begin{itemize}[leftmargin=*]
    \item \textbf{Lack of understanding in fundamental sample efficiency}: Previous works, especially in MFG setting, mainly consider the computational complexity, and therefore, they only focus on settings with strong structural assumptions, such as contractivity \citep{guo2019learning, xie2021learning}, monotonicity \citep{perrin2020fictitious, perolat2021scaling, elie2020convergence}, or population-independent dynamics \citep{mahajan2021reinforcement, geist2021concave}.
    As a result, the fundamental sample efficiency under general conditions is still an open problem.
    On the other hand, those structural assumptions also simplify the exploration process to some extent, so their algorithms can be hardly generalized beyond those assumptions.
    \item \textbf{Lack of understanding in function approximation setting}: Most previous literature only focuses on tabular setting \citep{guo2019learning, elie2020convergence}, and the sample complexity bounds depend on the number of states and actions, which are unacceptable when the state or action spaces are very large.
    To the best of our knowledge, the only work studying the sample complexity of MFRL in the function approximation setting is \citep{pasztor2021efficient}. However, \citet{pasztor2021efficient} only consider MFC setting, and are limited to near-deterministic transitions with sub-Gaussian noises, which cannot model the transition distributions with multiple modes.
\end{itemize}

Motivated by these limitations, we focus on the general model-based function approximation setting, where the state-action spaces can be arbitrarily large but we have access to a model class $\cM$ to approximate the dynamics of the Mean-Field system.
The key question we would like to address is:
\begin{center}
    \emph{What is the \textbf{fundamental} sample efficiency of Mean-Field RL with model function approximation?}
\end{center}
In contrast with previous literature (especially in MFG setting), by ``fundamental'', we aim to understand the sample efficiency \emph{with the minimal (most basic) assumptions} including realizability (Assump.~\ref{assump:realizability}) and Lipschitz continuity (Assump.~\ref{assump:lipschitz}) without additional strong structural assumptions like contractivity or monotonicty \citep{guo2019learning,perrin2020fictitious, perolat2021scaling, elie2020convergence}.
We treat them as fundamental assumptions because realizability ensures a good approximation exists so the learning is possible, and the Lipschitz assumption, as we will show later, guarantees the existence of the learning objective in MFG setting.
Moreover, we only consider the trajectory sampling model (Def.~\ref{def:collection_process}), which is much milder than the generative model assumptions \citep{guo2019learning}.

To address our key question, in Sec.~\ref{sec:MB_ED_MFRL} we first propose a new notion called Mean-Field Model-Based Eluder Dimension (MF-MBED). MF-MBED, generalized from Eluder Dimension in single-agent value approximation setting \citep{russo2013eluder}, characterizes the complexity of mean-field function classes including but not limited to tabular setting. We also provide concrete examples of mean-field model classes with low MF-MBED, such as (generalized) linear MF-MDP, near-deterministic transition with Gaussian noise, etc.
In Sec.~\ref{sec:O_MLE}, we develop sample efficient model-learning algorithms for MFRL based on Maximal Likelihood Estimation (MLE), which can achieve sample complexity only polynomial in MF-MBED for both MFC and MFG.

We highlight our main contributions in the following:
\begin{itemize}[leftmargin=*]
    \item We introduce a new notion called Mean-Field Model-Based Eluder Dimension (MF-MBED) to measure the complexity of any given Mean-Field model function class,
    and identify concrete examples exhibit low MF-MBED.
    \item For MFG setting, we propose the first MLE-based algorithm which is capable of addressing the exploration challenge while imposing minimal structural assumption.
    Further, we establish the first fundamental sample complexity result for function approximation setting, which only has polynomial dependence on MF-MBED, without explicit dependence on the number of states and actions.
    \item On the technical level, we establish MLE learning guarantees for Mean-Field setting, while previous results are limited to single-agent setting.
    Notably, the dependence on state density in transition function introduces unique challenges in analysis, which we overcome by establishing close connections between model class complexity, MLE error, and the MFC and MFG objectives.
\end{itemize}

\section{RELATED WORK}\label{sec:related_work}
In general, the theoretical understanding of MFRL in the finite horizon setting is still limited, especially in terms of statistical efficiency.
We present and compare with several lines of work in mean-field setting and defer the related works in single-agent or general multi-agent setting to Appx.~\ref{appx:additional_related_works}.

\paragraph{Finite-horizon Non-Stationary MFG}
The finite-horizon framework considered here is closely related to Lasry-Lions games \citep{perrin2020fictitious,perolat2021scaling,geist2021concave}, where continuous-time dynamics were analyzed without exploration considerations under monotonicity assumptions on rewards.
Most existing works consider additional structural assumptions like contractivity \citep{guo2019learning} monotonicity \citep{perolat2021scaling}.
We defer to Remark~\ref{remark:comparison_assumptions} for a more detailed comparison among these assumptions.
Besides, most previous literature \citep{elie2020convergence} requires a planning oracle that can return a trajectory for arbitrary state density, even if it can not be induced by any policy.
In contrast, our work focuses on understanding the fundamental exploration guarantees and identifying bottlenecks associated with finite-horizon MFC and MFG. 
Consequently, our findings extend beyond MFG/MFC scenarios that adhere to restrictive conditions.

\paragraph{Stationary MFG}
Alternative to the finite-horizon non-stationary MFG formulation, there exist works on the stationary MFG formulation
\citep{anahtarci2023q, xie2021learning, yardim2022policy, cui2021approximately}.
In this formulation,  the transition and reward functions at each step are conditioned on the stationary density rather than on the evolving density across time, which poses a limitation.
Furthermore, existing results typically require strong Lipschitz continuity assumptions as well as non-vanishing regularization \citep{anahtarci2023q, cui2021approximately}.

\paragraph{Statistical Efficiency Results for MFC}
In terms of statistical efficiency considerations, a related work \citep{pasztor2021efficient} analyzes the MFC setting in an information gain framework. Our results capture different learnable function classes.
Our low MF-MBED model classes can capture multi-modal transition distribution (e.g. the linear setting), while their algorithm and analysis are limited to near-deterministic transition with random noise (unimodal transition distribution).
Besides, our framework encompasses certain special cases in \citep{pasztor2021efficient},
as a result of our Prop.~\ref{prop:MF-MBED_det_trans_noise} and the equivalence between Eluder Dimension and Information Gain in RKHS space \citep{jin2021bellman}.
However, it is worthy noting a limitation in our approach: \citep{pasztor2021efficient} can cover the cases when the noise is sub-Gaussian besides pure Gaussian, as long as they can get access to the full information of the noise. 

\paragraph{Other MFG/MFC Settings}
There also exists a variety of different settings in which the MFGs formalism has been utilized, for instance in linear quadratic MFG \citep{guo2022entropy} and MFGs on graphs \citep{yang2018mean, gu2021mean}.
\citep{angiuli2022unified} studies a unified view of MFG and MFC, however, they do not take the evolution of density into consideration and do not provide guarantees for the non-tabular setting.
Several works on MFC also work on the lifted MDP where population state is observable \citep{carmona2019model}.
In our work, we do not assume the observability of the population.

\section{BACKGROUND}
\subsection{Setting and Notation}
We consider the finite-horizon Mean-Field Markov Decision Process (MF-MDP) specified by a tuple $M := (\mu_1,\cS,\cA,H,\mP_{T},\mP_{\textR})$. Here $\mu_1$ is the fixed initial distribution known to the learner, $\cS$ and $\cA$ are the state and action space.
Without loss of generality, we assume the state and action spaces are discrete but can be arbitrarily large.
We assume the state-action spaces are the same for each step $h$, i.e., $\cS_h=\cS$ and $\cA_h=\cA$ for all $h$. $\mP_{T}:=\{\mP_{T,h}\}_{h=1}^H$ and $r:=\{r_h\}_{h=1}^H$ are the transition and (normalized) deterministic reward function, with $\mP_{T,h}:\cS_h\times\cA_h\times\Delta(S_h)\rightarrow\Delta(S_{h+1})$ and $r_h:\cS_h\times\cA_h\times\Delta(S_h)\rightarrow [0,\frac{1}{H}]$. 
Without loss of generality, we assume that the reward function is known, but our techniques can be extended when it is unknown and a reward function class is available.
We use $M^*$ to denote the true model with transition function $\mP_{T^*}$.

In this paper, we only consider non-stationary Markov policy $\pi:=\{\pi_1,...,\pi_H\}$ with $\pi_h:\cS_h\rightarrow\Delta(\cA_h),~\forall~h\in[H]$.
Starting from the initial state $s_1 \sim \mu_1$ until the fixed final state $s_{H+1}$ is reached, the trajectory is generated by $\forall h\in[H]$:
\begin{align*}
    & a_h \sim \pi_h(\cdot|s_h),~s_{h+1} \sim \mP_{T,h}(\cdot|s_h,a_h,\mu^\pi_{M,h}), \\
    & r_h \sim r_h(s_h,a_h,\mu^\pi_{M,h}),~\mu^\pi_{M,h+1} = \Gamma^\pi_{M,h}(\mu^\pi_{M,h}),\numberthis\label{eq:MF_formulation}
\end{align*}
where:
\begin{align*}
    \Gamma^\pi_{M,h}(\mu_h)(\cdot) := \SumInt_{s_h,a_h} \mu_h(s_h)\pi(a_h|s_h) \mP_{T,h}(\cdot|s_h,a_h,\mu_h)
\end{align*}
and we use $\mu^\pi_{M,h}$ to denote the density induced by $\pi$ under model $M$ and $\Gamma^\pi_{M,h}:\Delta(\cS_h)\rightarrow\Delta(\cS_{h+1})$ is an mapping from densities in step $h$ to step $h+1$ under $M$ and $\pi$.
We will use bold font $\vecmu := \{\mu_1,...,\mu_H\}$ to denote the collection of density for all time steps.
Besides, we denote $V^{\pi}_{M,h}(\cdot;\vecmu)$ to be the value functions at step $h$ if the agent deploys policy $\pi$ in model $M$ conditioning on $\vecmu$, defined by:
\begin{align}
    &V^{\pi}_{M,h}(s_h;\vecmu) := \EE\Big[\sum_{\ph=h}^H r_\ph(s_\ph,a_\ph,\mu_\ph) \Big| \nonumber\\
    &\qquad\qquad \forall \tldh \geq h,~a_\tldh \sim \pi_\tldh,~s_{\tldh+1}\sim \mP_{T,\tldh}(\cdot|s_\tldh,a_\tldh,\mu_\tldh)\Big].\label{eq:value_function}
\end{align}
We use $J_M(\pi;\vecmu) := \EE_{s_1\sim\mu_1}[V^{\pi}_{M,1}(s_1;\vecmu)]$ to denote the expected return of policy $\pi$ in model $M$ conditioning on $\vecmu$.
When the policy is specified, we use $\vecmu^\pi_M := \{\mu^\pi_{M,1},...,\mu^\pi_{M,H}\}$ to denote the collection of mean fields w.r.t. $\pi$. 
We will omit $\vecmu$ and use $J_M(\pi)$ in shorthand when $\vecmu = \vecmu^\pi_M$.
For  simplicity, in the rest of the paper, we use 
$$
\EE_{\pi,M|\vecmu}[\cdot] := \EE\left[\cdot\Big|\substack{s_1\sim\mu_1 \\ \forall h\geq 1,\quad a_h \sim \pi_h(\cdot|s_h)\\ s_{h+1}\sim\mP_{T,h}(\cdot|s_h,a_h,\mu_h)}\right]
$$ 
as a shortnote of the expectation over trajectories induced by $\pi$ under transition $\mP_{T,h}(\cdot|\cdot,\cdot,\mu_h)$, and we omit the conditional density $\vecmu$ if $\vecmu = \vecmu^\pi_M$.
As examples, $V^\pi_{M,h}(s_h;\vecmu) = \EE_{\pi,M|\vecmu}[\sum_{h'=h}^H r(s_\ph,a_\ph,\mu_\ph)|s_h]$ and $J_M(\pi) = \EE_{\pi,M}[\sum_{h=1}^H r(s_\ph,a_\ph,\mu^\pi_{M,\ph})]$.

Given a measure space $(\Omega, \cF)$ and two probability measures $P$ and $Q$ defined on $(\Omega, \cF)$, we denote $\matDTV(P,Q)$ (or $\|P - Q\|_\TV$)$ := \sup_{A \in \cF} |P(A) - Q(A)|$ as the total variation distance, and denote $\mH(P,Q) := \sqrt{1 - \SumInt_x \sqrt{P(x)Q(x)}}$ as the Hellinger distance. In general, when $\Omega$ is countable, we have $\sqrt{2} \mH(P,Q) \geq \matDTV(P,Q) = \frac{1}{2} \|P-Q\|_1$, where $\|\cdot\|_1$ is the $l_1$-distance.

\paragraph{Mean-Field Control}
In MFC, similar to single-agent RL, we are interested in the optimality gap 
$
\cE_{\Opt}(\pi):= \max_\tpi J_{M^*}(\tpi;\vecmu^\tpi_{M^*}) - J_{M^*}(\pi;\vecmu^{\pi}_{M^*}),
$ 
and aim to find a policy $\hat\pi^*_\Opt$ to approximately minimize it:
\begin{equation}
    \cE_{\Opt}(\hat\pi^*_\Opt)\leq \epsilon.\label{eq:objective_MFC}
\end{equation}

\paragraph{Mean-Field Game}
In MFG, we instead want to find a Nash Equilibrium (NE) policy s.t., when all the agents follow that same policy, no agent tends to deviate from it for better policy value.
We denote $\Delta_M(\tpi, \pi) := J_M(\tpi;\vecmu^\pi_M) - J_M(\pi;\vecmu^\pi_M)$ given a model $M$, and denote 
$
\cE_\NE(\pi) := \max_\tpi \Delta_{M^*}(\tpi, \pi),
$
which is also known as the exploitability.
The NE in $M^*$ is defined to be the policy $\pi_{\NE}^*$ satisfying $\cE_\NE(\pi_{\NE}^*) = 0$.
In MFG, the objective is to find an approximate NE $\hat\pi^*_\NE$ such that:
\begin{equation}
    \cE_\NE(\hat\pi_{\NE}^*) \leq \epsilon.\label{eq:objective_MFG}
\end{equation}
\subsection{Assumptions}
In this paper, we consider the general function approximation setting, where the learner has access to a model class $\cM$ with finite cardinality ($|\cM| < +\infty$), satisfying the following assumptions.
\begin{assumption}[Realizability]\label{assump:realizability}
    $M^* \in \cM$.
\end{assumption}
\begin{assumption}[Lipschitz Continuity]\label{assump:lipschitz}
    For arbitrary $h\in[H],s_h\in\cS,a_h\in\cA$ and arbitrary valid density $\mu_h,\mu_h'\in\Delta(\cS)$, and arbitrary model $M:=(\mP_T,r) \in \cM$, we have:
    \begin{align}
        &\|\mP_{T,h}(\cdot|s_h,a_h,\mu_h) - \mP_{T,h}(\cdot|s_h,a_h,\mu_h')\|_\TV \nonumber\\
        &\qquad\qquad\qquad\qquad\qquad\qquad \leq L_T \|\mu_h-\mu_h'\|_\TV, \label{eq:Lip_model_H}\\
        &\|r_h(s_h,a_h,\mu_h) - r_h(s_h,a_h,\mu_h')\|_\TV \nonumber\\
        &\qquad\qquad\qquad\qquad\qquad\qquad \leq L_r \|\mu_h-\mu_h'\|_\TV. \label{eq:Lip_rew_l1}
    \end{align}
\end{assumption}
\begin{remark}[Comparison with Previous Structural Assumptions in MFG Setting]\label{remark:comparison_assumptions}
    The most common structural assumptions in previous finite-horizon MFG setting include contractivity \citep{guo2019learning} and monotonicity \citep{perrin2020fictitious,perolat2021scaling}.
    The contractivity is stronger than ours since in general it requires good smooth conditions of $\mP_T$ w.r.t. states, actions, state densities \citep{yardim2022policy}.
    The monotonicity instead considers the reward structures, and it is stronger in that it assumes the transition is independent of density (i.e. $L_T=0$), so that the dynamics of the system reduces to single-agent RL.
\end{remark}
In Prop.~\ref{prop:exist_MFG_NE_informal} below, we show that Assump.~\ref{assump:lipschitz} implies the existence of a Nash Equilibrium.
A similar existence result has been established in previous literature \citep{saldi2018markov} under the same conditions as our Prop.~\ref{prop:exist_MFG_NE_formal}. Our contribution here is a different proof based on the conjugate function and non-expansiveness of the proximal point operator. Moreover, \citep{saldi2018markov} studied infinite-horizon MDP with discounted reward, which is different from our setting.
\begin{proposition}[Existence of NE in MFG; Informal Version of Prop.~\ref{prop:exist_MFG_NE_formal}]\label{prop:exist_MFG_NE_informal}
    For every MF-MDP with discrete $\cS$ and $\cA$, satisfying Assump.~\ref{assump:lipschitz}, there exists at least one NE policy.
\end{proposition}

Besides, we formalize the trajectory sampling model in the following.
Our trajectory sampling model is 
much weaker than the assumption in \citep{guo2019learning,elie2020convergence}, which requires a planning oracle that can return a trajectory conditioning on arbitrary (even unachievable) density.
\begin{definition}[Trajectory Sampling Model]\label{def:collection_process}
We assume the environment consists of an extremely large number of agents and a central controller (our algorithm/learner), and there is a representative agent $\texttt{Agent}$, whose observation we can receive. 
The central controller can compute an arbitrary policy tuple ($\tilde\pi$, $\pi$) (here $\pi$ and $\tpi$ are not necessarily the same), distribute $\tilde\pi$ to $\texttt{Agent}$ but $\pi$ to the others, and receive the trajectory of $\texttt{Agent}$ following $\tpi$ under $\mP_{T^*,h}(\cdot|\cdot,\cdot,\mu^\pi_h)$ and $\mP_{r,h}(\cdot|\cdot,\cdot,\mu^\pi_h)$.

\end{definition}

\section{MEAN-FIELD MODEL-BASED ELUDER DIMENSION}\label{sec:MB_ED_MFRL}
In order to avoid explicit dependence on the number of states and actions, we focus on the intrinsic complexity of the function class instead of the complexity of the state and action spaces.
In this section, we introduce the notion of Mean-Field Model-Based Eluder Dimension (MF-MBED), which characterizes the complexity of an arbitrary model function class $\cM$ via the length of the longest state-action- state density sequence $\{(s^i,a^i,\mu^i)\}_{i\in[n]}$ such that each $(s^i,a^i,\mu^i)$ introduces new ``information'' of $\cM$ given the information revealed by previous data $\{(s^t,a^t,\mu^t)\}_{t=1,...,i-1}$.
In the following, we introduce the formal definition.

\begin{definition}[$\alpha$-weakly-$\epsilon$-independent sequence]\label{def:eps_independent}
    Denote $\cX:=\cS\times\cA\times\Delta(\cS)$ to be the joint space of state, action, and state density. 
    Let $\matD:\Delta(\cS)\times\Delta(\cS)\rightarrow [0,C]$ be a distribution distance measure bounded by some constant $C$.
    Given a function class $\cF\subset\{f:\cX\rightarrow\Delta(\cS)\}$, a fixed $\alpha \geq 1$ and a sequence of data points $x_1,x_2,...,x_n \in \cX$, we say $x$ is $\alpha$-weakly-$\epsilon$-independent of $x_1,...,x_n$ w.r.t. $\cF$ and $\matD$ if there exists $f_1,f_2\in\cF$ such that $\sum_{i=1}^{n} \matD^2(f_1, f_2)(x_i) \leq \epsilon^2$ but $\matD(f_1, f_2)(x) > \alpha\epsilon$.
\end{definition}
\begin{definition}[The longest $\alpha$-weakly-$\epsilon$-independent sequence]
    We use $\dimE_\alpha(\cF,\matD,\epsilon)$ to denote the longest sequence $\{x_i\}_{i=1}^n \in \cX$, such that for some $\epsilon' \geq \epsilon$, $x_{i}$ is $\alpha$-weakly-$\epsilon'$-independent of $\{x_1,...,x_{i-1}\}$ for all $i\in[n]$ w.r.t. $\cF$ and $\matD$.
\end{definition}
\begin{definition}[Model-Based Eluder-Dimension in Mean-Field RL]\label{def:Eluder_Dim}
    Given a model class $\cM$, $\alpha \geq 1$ and $\epsilon > 0$, the Model-Based Eluder Dimension in MFRL (abbr. MF-MBED) of $\cM$ is defined as:
    \begin{align*}
        \textstyle \dimE_\alpha&(\cM,\epsilon) := \\
        &\max_{h\in[H]} \min_{\matD\in\{\matDTV,\mH\}} \dimE_\alpha(\cM_h,\matD,\epsilon).\numberthis\label{eq:MB_ED_MFRL}
    \end{align*}
\end{definition}
We only consider $\matD$ to be $\matDTV(P,Q)$ or $\mH(P,Q)$, mainly because of our MLE-based loss function. With slight abuse of notation, $\cM$ (or $\cM_h$) here refers to the collection of transition functions of models in $\cM$.
The main difference comparing with value function approximation setting \citep{russo2013eluder,jin2021bellman} is that, because the output of model functions are distributions instead of scalar, we use distance measure to compute the model prediction difference.
Furthermore, we use $\alpha\epsilon$  instead of $\epsilon$ in the threshold, which does not lead to a fundamentally different complexity measure, but simplifies the process to absorb some practical examples into our framework.
Also note that $\dimE_{\alpha_1}(\cF,\epsilon) \leq \dimE_{\alpha_2}(\cF,\epsilon)$ for $\alpha_1 \geq \alpha_2$, because any $\alpha_1$-weakly-$\epsilon$-independent sequence must be $\alpha_2$-weakly-$\epsilon$-independent.

\paragraph{Comparison with Previous Work Regarding Eluder Dimension}
Multiple previous work considers using independent sequences to characterize the complexity of function classes, but most of them focus on  value function approximation in the single-agent setting \citep{russo2013eluder,jin2021bellman}.
To our knowledge, only limited work \citep{osband2014model,levy2022eluderbased} has studied Eluder Dimension for model-based function approximation even in single-agent setting.
\citep{osband2014model} additionally assumed given two transition distributions in the function class, the difference between their induced future value function is Lipschitz continuous w.r.t. their mean difference, which is restrictive.
More recently, \citep{levy2022eluderbased} presented extension of MBED to general bounded metrics, however, their results still depend on the number of states actions, and concrete examples with low MBED are not provided.

\subsection{Concrete Examples with Low MF-MBED}
Next, we introduce some concrete examples with low MF-MBED, and defer formal statements and their proofs to Appx.~\ref{appx:examples_MF-MBED}.
The first one is generalized from the linear MDP in single-agent RL \citep{jin2020provably}.
\begin{proposition}[Low-Rank MF-MDP with Known Representation; Informal Version of Prop.~\ref{prop:MF-MBED_Linear_MFMDP_formal}]\label{prop:MF-MBED_Linear_MFMDP}
    Given a feature $\phi:\cS\times\cA\times\Delta(\cS)\rightarrow\mR^d$ and a function class $\Psi$, the model class $\cP_\Psi := \{\mP_\psi|\mP_\psi(s'|s,a,\mu):=\phi(s,a,\mu)\trans\psi(s'),~ \psi\in \Psi\}$ has $\dimE_\alpha(\cP_\Psi,\matDTV,\epsilon) = \tilde{O}(d)$ for $\alpha \geq 1$.
\end{proposition}
In Appx.~\ref{appx:examples_MF-MBED}, we also include a linear mixture type model, and other more general examples, such as, kernel MF-MDP and the generalized linear MF-MDP.
We defer to appendix for more discussions about our technical novelty here.
Basically, since the output of the model function is a probability distribution rather than a scalar, we utilize data-dependent sign functions to establish the connections between the TV-distance of distribution predictions and the error of next state features.

The second example considers deterministic transition with random noise, in order to accommodate the function class in \citep{pasztor2021efficient} (see a detailed comparison in Sec.~\ref{sec:related_work}).
Here we consider the Hellinger distance because given two Gaussian distribution $P\sim\cN(\mu_P,\Sigma)$ and $Q\sim\cN(\mu_Q,\Sigma)$ with the same covariance, $\mH(P,Q) = 1 - \exp(-\frac{1}{8}\|\mu_P - \mu_Q\|_{\Sigma^{-1}}^2)$.
Therefore, with the connection between $\mH(P,Q)$ and the $l_2$-distance between their mean value, we are able to subsume more important model classes into low MF-MBED framework.
\begin{restatable}{proposition}{PropExampleHD}[Deterministic Transition with Gaussian Noise]\label{prop:MF-MBED_det_trans_noise}
    Suppose $\cS \subset \mR^d$. Given a function class $\cG\subset\{g|g:\cS\times\cA\times\Delta(\cS)\times\mN^*\rightarrow\mR\}$ and convert it to $\cF_\cG:=\{f_g| f_g(\cdot,\cdot,\cdot) := [g(\cdot,\cdot,\cdot,1),...,g(\cdot,\cdot,\cdot,d)]\trans\in\mR^d,~g\in\cG\}$. Consider the model class $
        \cP_\cG := \{\mP_f | \mP_f(\cdot|s,a,\mu)\sim f(s,a,\mu) + \cN(0,\Sigma), f\in\cF_\cG\}
    $, where 
    $\Sigma := \diag(\sigma,...,\sigma)$.
    For $\epsilon\leq 0.3$, we have $\dimE_{\sqrt{2}}(\cP_\cG,\mH,\epsilon) \leq \overline{\dimE}(\cF_\cG,4\sigma\epsilon)$, $\dimE_{\sqrt{2d}}(\cP_\cG,\mH,\epsilon) \leq \overline{\dimE}(\cG,4\sigma\epsilon)$, where $\overline{\dimE}$ is the Eluder Dimension for scalar or vector-valued functions \citep{russo2013eluder,osband2014model}.
\end{restatable}
\begin{remark}
    In the Gaussian example above, the state space has to be continuous since the Gaussian distribution is defined in continuous space.
    Although in this paper we only consider the discrete state space, our results are based on $\TV$-distance and similar results can be established when the state space is continuous.
\end{remark}
\section{LEARNING IN MEAN-FIELD RL: AN MLE APPROACH}\label{sec:O_MLE}
In this section, we develop a general Maximum Likelihood Estimation (MLE)-based learning framework for both MFC and MFG and show that given model classes with low MF-MBED, these MFRL problems are indeed tractable in sample complexity.

\subsection{Main Algorithm}\label{sec:main_alg_results}
Since the MLE-based model learning for MFC and MFG share some similarities, we unify them in one algorithm and use \red{MFC} and \blue{MFG} to distinguish algorithm steps for these two different objectives.
Our main algorithm is presented in Alg.~\ref{alg:MLE_MB_Alg}, where we omit the usage of rewards in learning to avoid redundancy in analysis.

The algorithm includes two parts: policy selection (Line \ref{line:SS}-\ref{line:SE}) and data collection (Line \ref{line:DS}-\ref{line:DE}).
In each iteration $k$, we fit the model with data $\cZ^1,...,\cZ^k$ collected so far and construct a model confidence set $\hat\cM^k$. The confidence level is carefully chosen, so that with high probability, we can ensure $M^* \in \hat\cM^k$ for all $k$.

In MFC, similar to the single-agent setting, we pick $\pi^{k+1}$ to be the policy achieving the maximal total return among models in the confidence set, and then use it to collect new samples for exploration. 
In the end, we use \texttt{Regret2PAC} conversion algorithm (Alg.~\ref{alg:Regret2PAC}, deferred to Appx.~\ref{appx:proofs_for_MFC}) to select policy.

For MFG, the learning process is slightly more complicated. For the policy selection part, we compute two policies. We first randomly pick $M^{k+1}$ from $\hat\cM^k$, and compute its equilibrium policy $\pi^{k+1}$ to be our guess for the equilibrium of the true model $M^*$.
Next, we find a model $\tM^{k+1}$ and an adversarial policy $\tpi^{k+1}$, which result in an optimistic estimation for $\cE_\NE(\pi^{k+1})$.
Besides, for the data collection part, in addition to the trajectories generated by deploying $\pi^{k+1}$, we also collect trajectories sampled by policy $\tpi^{k+1}$ conditioning on the density induced by $\pi^{k+1}$.
As we will explain in Lem.~\ref{lem:exploitability_diff}, those additional samples are necessary to control the estimation error of exploitability.
Finally, we return the policy with the minimal optimistic exploitability among $\{\pi^{k+1}\}_{k=1}^K$.

\subsection{Main Results on Statistical Efficiency}
We state our main results below, which establish the sample complexity of learning MFC/MFG in Alg.~\ref{alg:MLE_MB_Alg}. The formal version (Thm.~\ref{thm:MFC_main_full} and Thm.~\ref{thm:MFG_main_full}) and the proofs are deferred to Appx.~\ref{appx:proofs_for_MFRL}.

\begin{algorithm}
    \textbf{Input}: Model function class $\cM$; $\epsilon,\delta,K$.\\
    \textbf{Initialize}: Randomly pick $\pi^1$ and $\tpi^1$; $\cZ^k \gets \{\},~\forall k\in[K]$.\\
    \For{$k=1,2,...,K$}{
        \For{$h=1,...,H$}{\label{line:DS}
            Sample $z^k_h := \{s_h^k, a_h^k, s_{h+1}'^k\}$ with $(\pi^k,\pi^k)$; $\cZ^k \gets \cZ^k \cup z^k_h$.\\
            \lIf{\blue{\MFG}}{
                Sample $\tz_h^k := \{\ts_h^k, \ta_h^k, \ts_{h+1}'^k\}$ with ($\tpi^k$, $\pi^k$); $\cZ^k \gets \cZ^k \cup \tz_h^k$
            }
        }\label{line:DE}
        For each $M\in\cM$, define: \label{line:SS}
            \begin{align*}
                l_{\MLE}^{k}(M) :=& \sum_{i=1}^k\sum_{h=1}^H \log  \mP_{T,h}(s_{h+1}'^i|s_h^i,a_h^i,\mu^{\pi^{i}}_{M,h}) \\
                &+ \underbrace{\sum_{i=1}^k\sum_{h=1}^H\log  \mP_{T,h}(\ts_{h+1}'^i|\ts_h^i,\ta_h^i,\mu^{\pi^{i}}_{M,h})}_{\text{\blue{MFG}~only}}.
            \end{align*}
        \label{line:MLE_loss}\\
            $
            \hat\cM^k \gets \{M \in \cM |
            l_{\MLE}^k(M)\geq \max_{M\in\cM} l_{\MLE}^k(M) - \log\frac{2|\cM|KH}{\delta}\},
            $\label{line:confidence_set}\\
        \lIf{\red{\MFC}}{
            $\pi^{k+1}, M^{k+1} \gets \arg\max_{\pi, M\in\hat\cM^k} J_{M}(\pi;\vecmu^\pi_M)$
        }
        \If{\blue{\MFG}}{
            Randomly pick $M^{k+1}$ from $\hat\cM^k$; \\
            Find a NE of $M^{k+1}$ denoted as $\pi^{k+1}$.\\
            $\tpi^{k+1}, \tM^{k+1} \gets \arg\max_{\tpi; M\in\hat\cM^k} \Delta_{M}(\tpi, \pi^{k+1})$.
        }\label{line:SE}
    }
    \lIf{\red{\MFC}}{
        \Return $\hat\pi^*_{\Opt} \gets \text{Regret2PAC}(\{\pi^{k+1}\}_{k=1}^K, \epsilon, \delta)$
    }
    \lIf{\blue{\MFG}}{
        \Return $\hat\pi^*_\NE \gets \pi^{k^*_\NE}$ with $k^*_\NE \gets \min_{k\in[K]} \Delta_{\tM^{k+1}}(\tpi^{k+1}, \pi^{k+1})$
    }
    \caption{A General Maximal Likelihood Estimation-Based Algorithm for Mean-Field RL}\label{alg:MLE_MB_Alg}
\end{algorithm}

\begin{theorem}[Main Results (Informal)]\label{thm:short_main_results_MFC_MFG}
    Under Assump.\ref{assump:realizability},~\ref{assump:lipschitz}, by running Alg.~\ref{alg:MLE_MB_Alg} with \footnote{We omit log-dependence on $\epsilon,\delta,\dimE,|\cM|,H$ and Lipschitz factors in $\tilde{O}$, and it's same for Thm.~\ref{thm:short_main_results_MFC_MFG_contractivity}}
    \begin{align*}
        K = \tilde{O}\Big(&H^2(1+L_rH)^2(1+L_TH)^2 \\
        &\cdot\Big(\frac{(1+L_T)^H - 1}{L_T}\Big)^2 \frac{\dimE_\alpha(\cM,\epsilon_0)}{\epsilon^2}\Big)
    \end{align*}
    with $\epsilon_0 = O(\frac{L_T \epsilon}{\alpha H(1+L_rH)(1+L_TH)((1+L_T)^H - 1)})$,
    \begin{itemize}
        \item for \red{MFC}, after consuming HK trajectories in Alg.~\ref{alg:MLE_MB_Alg} and additional $O(\frac{1}{\epsilon^2}\log^2\frac{1}{\delta})$ trajectories in Alg.~\ref{alg:Regret2PAC}, w.p. $1-5\delta$, we have $\cE_{\Opt}(\hat\pi^*_\Opt)\leq \epsilon$.

        \item for \blue{MFG}, after consuming $2HK$ trajectories, w.p. $1-3\delta$, we have $\cE_\NE(\hat\pi^*_\NE) \leq \epsilon$.
    \end{itemize}
\end{theorem}

As stated in our main results, both MFC and MFG can be sample efficiently solved as long as the model function classes has low MF-MBED. 
Also note that our sample complexity bounds are universal since MF-MBED can be defined for arbitrary mean-field model function classes.
As reflected by Thm.~\ref{thm:short_main_results_MFC_MFG}, comparing with previous literature requiring additional structural assumptions \citep{guo2019learning, xie2021learning,perrin2020fictitious, perolat2021scaling, elie2020convergence}, our algorithm can tackle the exploration challenges in more general setting where only realizability and Lipschitz continuity assumptions are satisfied.

As for the dependence of Lipschitz factors, similar exponential dependence of $L_T$ has been reported in previous literature \citep{pasztor2021efficient}. 
We conjecture that such exponential dependence reflects some fundamental difficulties if we only make minimal assumptions.
In the following, we show how to avoid exponential terms under additional contraction assumptions.
\begin{assumption}[Contraction Operator]\label{assump:contraction}
    For arbitrary $h$, and arbitrary valid density $\mu_h,\mu_h' \in \Delta(\cS)$, and arbitrary model $M:=(\mP_T,r) \in \cM$, there exists $L_\Gamma < 1$, such that,
    \begin{align*}
        \forall \pi,~\|\Gamma^\pi_{M,h}(\mu_h) - \Gamma^\pi_{M,h}(\mu_h')\|_{\TV} \leq L_\Gamma\|\mu_h - \mu_h'\|_{\TV}.
    \end{align*}
\end{assumption}
where $\Gamma^\pi_{M,h}(\mu_h)$ is defined in Eq.~\eqref{eq:MF_formulation}. According to \citep{yardim2022policy}, Assump.~\ref{assump:contraction} is implied by some Lipschitz continuous assumption on the transition function w.r.t. the Dirac distance $d(s,s') := \mathbb{I}[s \neq s']$ (at least when $\cS$ and $\cA$ are countable).

We summarize the main results given Assump.~\ref{assump:contraction} below, where the formal versions are also included in Thm.~\ref{thm:MFC_main_full} and Thm.~\ref{thm:MFG_main_full}. 
Comparing with Thm.~\ref{thm:short_main_results_MFC_MFG}, the sample complexity exhibits only polynomial dependence on Lipschitz factors.
\begin{theorem}[Main Results (Informal)]\label{thm:short_main_results_MFC_MFG_contractivity}
    Under Assump.\ref{assump:realizability},~\ref{assump:lipschitz} and~\ref{assump:contraction}, by choosing:
    \begin{align*}
        K = \tilde{O}\Big(H^2(1+L_r&H)^2(1 + L_TH)^2 \\
        &\Big(1 + \frac{L_T}{1 - L_\Gamma}\Big)^2 \frac{\dimE_\alpha(\cM,\epsilon_0)}{\epsilon^2}\Big),
    \end{align*} 
    with $\epsilon_0 = O(\frac{\epsilon}{\alpha H(1 + L_TH)(1+L_rH)} (1+\frac{L_T}{1-L_\Gamma})^{-1})$
    \begin{itemize}
        \item for \red{MFC}, after consuming HK trajectories in Alg.~\ref{alg:MLE_MB_Alg} and additional $O(\frac{1}{\epsilon^2}\log^2\frac{1}{\delta})$ trajectories in Alg.~\ref{alg:Regret2PAC}, w.p. $1-5\delta$, we have $\cE_{\Opt}(\hat\pi^*_\Opt)\leq \epsilon$.

        \item for \blue{MFG}, after consuming $2HK$ trajectories, w.p. $1-3\delta$, we have $\cE_\NE(\hat\pi^*_\NE) \leq \epsilon$.
    \end{itemize}
\end{theorem}

\subsection{Generalization to Infinite Model Class and Continuous Setting}
When the model class $\cM$ contains infinite model candidates, our results can be generalized by the following steps. Firstly, we can find an $\epsilon$-cover $\mathcal{M}_\epsilon$ w.r.t. the distance $d(\cdot,\cdot)$ defined by:
\begin{align*}
    d(M,M'):=&\\
    \max_{\pi,\mu}\max\{&\mathbb{E}_{\pi,M|\mu}[\sum_{h}\log \frac{\mathbb{P}_{M,h}(s'_h|s_h,a_h,\mu_h)}{\mathbb{P}_{M',h}(s'_h|s_h,a_h,\mu_h)}], \\
    & \mathbb{E}_{\pi,M'|\mu}[\sum_{h}\log \frac{\mathbb{P}_{M',h}(s'_h|s_h,a_h,\mu_h)}{\mathbb{P}_{M,h}(s'_h|s_h,a_h,\mu_h)}]\}
\end{align*}
which aligns with our MLE method.
For $\cM_\epsilon$, Assump.~\ref{assump:realizability} may not hold because of the discretization, but there exists a model $\hM^*\in\cM_\epsilon$ close to $M^*$ under distance $d$ by definition.
Then, all we need to do is to revise line~\ref{line:confidence_set} of Alg.~\ref{alg:MLE_MB_Alg} to 
\begin{align*}
\hcM^k \gets &\{M \in \mathcal{M}|l_{MLE}^k(M)\geq \\
& \max_{M\in\mathcal{M}} l_{MLE}^k(M) - \log\frac{2|\mathcal{M}|KH}{\delta}-O(\epsilon)\},
\end{align*}
where we have additional tolerance at level $O(\epsilon)$.
Then we can extend our sample complexity results and it will only depend on $\log |\mathcal{M}_\epsilon|$ instead.
Besides, assuming low log covering number is common in previous literature, e.g. Def. 13 in \citep{jiang2017contextual}, Def. 3 in \citep{jin2021bellman}, etc.

Our algorithm, analysis, and the notion of "Model-Based Eluder Dimension" can be extended to the case with continuous compact $\mathcal{S,A}$ spaces with minor modifications ({such as replacing $\sum_{s,a}$ with $\int dsda$ and $l_1$ distance with total-variation distance}). The only exception is Prop. 3.2, which relies on fixed point theorem in finite space, but it is reasonable to assume the existence of Nash equilibrium when the state and action spaces are continuous.

\section{PROOF SKETCH}\label{sec:analysis_Obj_and_MPE}
In this section, we provide proof sketch to establish Thm.~\ref{thm:short_main_results_MFC_MFG}.
Intuitively, the proofs consist of two parts. 
Firstly, we provide an upper bound for the accumulative model prediction error by the MF-MBED, which we further connect with our learning objective in the second step.

\paragraph{Step 1: Upper Bound Model Prediction Error with MF-MBED}
First of all, in Thm.~\ref{thm:MLE_Gaurantee} below, we show that, with high probability, models in $\hat\cM^k$ predict well under the distribution of data collected so far.
We defer the proof to Appx.~\ref{appx:proof_for_MLE}.
\begin{restatable}{theorem}{ThmMLE}[Guarantees for MLE]\label{thm:MLE_Gaurantee}
    By running Alg.~\ref{alg:MLE_MB_Alg} with any $\delta \in (0,1)$, with probability $1-\delta$, for all $k\in[K]$, we have $M^* \in \hat{\cM}^k$; for each $M\in\hat{\cM}^k$ with transition $\mP_T$ and any $h\in[H]$:
    \begin{align*}
        \sum_{i=1}^k \EE_{\pi^i,M^*}&[\mH^2(\mP_{T,h}(\cdot|s_h^i,a_h^i,\mu^{\pi^i}_{M,h}),\\
        & \mP_{T^*,h}(\cdot|s_h^i,a_h^i,\mu^{\pi^i}_{M^*,h}))]  \leq 2 \log (\frac{2|\cM|KH}{\delta}).
    \end{align*}
    Besides, for MFG branch, we additionally have:
    \begin{align*}
        \sum_{i=1}^k &\EE_{\tpi^i,M^*|\vecmu^{\pi^i}_{M^*}}[\mH^2(\mP_{T,h}(\cdot|\ts_h^i,\ta_h^i,\mu^{\pi^i}_{M,h}),\\
        &\mP_{T^*,h}(\cdot|\ts_h^i,\ta_h^i,\mu^{\pi^i}_{M^*,h}))] \leq 2 \log (\frac{2|\cM|KH}{\delta}).
    \end{align*}
\end{restatable}

The key difficulty in Mean-Field setting is the dependence of density in transition function.
Since we do not know $\mu^\pi_{M^*,h}$, in Line~\ref{line:MLE_loss} in Alg.~\ref{alg:MLE_MB_Alg}, we compute the likelihood conditioning on $\mu^\pi_{M,h}$, which is accessible for each $M$.
Therefore, in Thm.~\ref{thm:MLE_Gaurantee}, we can only guarantee $M$ aligns with $M^*$ conditioning on their own density $\vecmu^\pi_M$ and $\vecmu^\pi_{M^*}$, respectively.
However, to ensure low MF-MBED can indeed capture important practical models, the MF-MBED in Def.~\ref{def:Eluder_Dim} is established on shared density, which is also the main reason we additional consider Hellinger distance in Assump.~\ref{assump:lipschitz}.
To close this gap, in Thm.~\ref{thm:model_diff_full} below, we present how the model difference conditioning on the same or different densities can be converted to each other. The proof is defered to Appx.~\ref{appx:proofs_for_MFRL}.

\begin{restatable}{theorem}{ThmModelDiff}[Model Difference Conversion]\label{thm:model_diff_full}
    Given two arbitrary model $M = (\cS,\cA,H,\mP_{T},\mP_r)$ and $\tM=(\cS,\cA,H,\mP_{\tT},\mP_r)$, and arbitrary policy $\pi$, under Assump.~\ref{assump:lipschitz}, we have:
    \begin{align*}
        \EE_{\pi,M}[\sum_{h=1}^{H} &\|\mP_{T,h}(\cdot|s_h,a_h,\mu^{\pi}_{M,h})  -\mP_{\tT,h}(\cdot|s_h,a_h,\mu^{\pi}_{M,h}) \|_\TV]  \\
        \leq & (1 + L_T H)\EE_{\pi,M}[\sum_{h=1}^H \|\mP_{T,h}(\cdot|s_h,a_h,\mu^\pi_{M,h}) \\
        &- \mP_{\tT,h}(\cdot|s_h,a_h,\mu^\pi_{\tM,h})\|_\TV],\numberthis\label{eq:model_diff_1}
    \end{align*}
    and
    \begin{align*}
        \EE_{\pi,M}[\sum_{h=1}^{H} & \|\mP_{T,h}(\cdot|s_h,a_h,\mu^{\pi}_{M,h}) - \mP_{\tT,h}(\cdot|s_h,a_h,\mu^{\pi}_{\tM,h}) \|_\TV] \\
        \leq & \EE_{\pi,M}[\sum_{h=1}^{H} (1+L_T)^{H-h}\| \mP_{T,h}(\cdot|s_h,a_h,\mu^{\pi}_{M,h}) \\
        &- \mP_{\tT,h}(\cdot|s_h,a_h,\mu^{\pi}_{M,h})\|_\TV].\numberthis\label{eq:model_diff_2}
    \end{align*}
\end{restatable}

In the final lemma, we show if the model predicts well in history, then the growth rate of the accumulative error on new data can be controlled by MF-MBED. We defer the proof to Appx.~\ref{appx:ED_2_Regret}.
\begin{restatable}{lemma}{LemEDtoReg}\label{lem:Eluder_Bound}
    Under the condition as Def.~\ref{def:eps_independent}, consider a fixed $f^* \in \cF$, and suppose we have a sequence $\{f_k\}_{k=1}^K \in \cF$ and $\{x_k\}_{k=1}^K \subset \cS\times\cA\times\Delta(\cS)$, s.t., for all $k\in[K]$, $\sum_{i=1}^{k-1} \matD^2(f_k,f^*)(x_i) \leq \beta$, then for any $\epsilon > 0$, we have $\sum_{k=1}^K \matD(f_k,f^*)(x_k) = O(\sqrt{\beta K \dimE_\alpha(\cM, \epsilon)} + \alpha K \epsilon)$.
\end{restatable}

\paragraph{Step 2: Relating Learning Objectives with Model Prediction Error}
First of all, we provide the simulation lemma for Mean-Field Control setting.
\begin{restatable}{lemma}{LemValDecomp}[Value Difference Lemma for MFC]\label{lem:value_decomposition}
    Given an arbitrary model $M$ with transition function $\mP_T$, and an arbitrary policy $\pi$, under Assump.~\ref{assump:lipschitz}, we have:
    \begin{align*}
        &|J_{M^*}(\pi) - J_M(\pi)| \leq \EE_{\pi,M^*}[\sum_{h=1}^H (1+L_rH)\\
        &\cdot \|\mP_{T^*,h}(\cdot|s_h,a_h,\mu^\pi_{{M^*},h})  - \mP_{T,h}(\cdot|s_h,a_h,\mu^\pi_{M,h})\|_\TV].
    \end{align*}
\end{restatable}

By Thm.~\ref{thm:MLE_Gaurantee} and Eq.~\eqref{eq:model_diff_1} in Thm.~\ref{thm:model_diff_full}, with high probability, all the models in $\hat{\cM}_k$ will agrees with each other on the dataset $D^k$ conditioning on the same density $\mu^{\pi^1}_{M^*},...,\mu^{\pi^k}_{M^*}$. 
On good concentration events, the condition for Lem.~\ref{lem:Eluder_Bound} is satisfied, and as a result of Thm.~\ref{lem:value_decomposition} and Eq.~\eqref{eq:model_diff_2}, we can upper bound the accumulative sub-optimal gap $\sum_{k=1}^K \cE_{\Opt}(\pi^{k+1})$. With the regret to PAC convertion process in Alg.~\ref{alg:Regret2PAC}, we can establish the sample complexity guarantee in Thm.~\ref{thm:short_main_results_MFC_MFG}.

For MFG, we first provide an upper bound for $\cE_\NE(\pi^{k+1})$. On the event of $M^* \in \hat\cM^{k+1}$, we have:
\begin{align*}
    \cE_\NE(\pi^{k+1}) & = \max_\pi \Delta_{M^*}(\pi, \pi^{k+1}) \leq \Delta_{\tM^{k+1}}(\tpi^{k+1}, \pi^{k+1}) \\
    \leq& \Delta_{\tM^{k+1}}(\tpi^{k+1}, \pi^{k+1}) - \Delta_{M^{k+1}}(\tpi^{k+1}, \pi^{k+1}) \\
    \leq & |\Delta_{\tM^{k+1}}(\tpi^{k+1}, \pi^{k+1}) - \Delta_{M^*}(\tpi^{k+1}, \pi^{k+1})| \\
    & + |\Delta_{M^*}(\tpi^{k+1}, \pi^{k+1}) - \Delta_{M^{k+1}}(\tpi^{k+1}, \pi^{k+1})|.
\end{align*}

where the first inequality is because of optimism, and the second one is because $\pi^{k+1}$ is the equilibirum of $M^{k+1}$. 
Next, we provide a key lemma to upper bound the RHS.
\begin{restatable}{lemma}{LemExploitDiff}[Value Difference Lemma for MFG]\label{lem:exploitability_diff}
    Given two arbitrary model $M$ and $\tM$, and two policies $\pi$ and $\tpi$, we have:
    \begin{align*}
        &|\Delta_M(\tilde\pi,\pi) - \Delta_\tM(\tilde\pi,\pi)| \\
        \leq & \EE_{\tpi,M|\vecmu^\pi_M}[\sum_{h=1}^H \|\mP_{T,h}(\cdot|s_h,a_h,\mu^\pi_{M,h}) \\
        &\qquad\qquad\qquad - \mP_{\tT,h}(\cdot|s_h,a_h,\mu^\pi_{\tM,h})\|_\TV]\\
        & + (2L_rH + 1) \EE_{\pi,M}[\sum_{h=1}^H  \|\mP_{T,h}(\cdot|s_h,a_h,\mu^\pi_{M,h}) \\
        &\qquad\qquad\qquad - \mP_{\tT,h}(\cdot|s_h,a_h,\mu^\pi_{\tM,h})\|_\TV].\numberthis\label{eq:exploitability_diff}
    \end{align*}
\end{restatable}
As we can see, to control the exploitability, we require the model can predict well on the data distribution induced by both $\pi^{k+1}$ and $\tpi^{k+1}$ conditioning on $\vecmu^{\pi^{k+1}}_{M^*}$, which motivates our formulation of Def.~\ref{def:collection_process}.
By combining Lem.~\ref{lem:exploitability_diff} and theorems in the first part, we finish the proof.

\paragraph{A Different Model Difference Conversion under Assump.~\ref{assump:contraction}}
In the following theorem, we provide a model difference conversion conditioning on different densities under Assump.~\ref{assump:contraction} to replace Eq.~\eqref{eq:model_diff_2} in Thm.~\ref{thm:model_diff_full}, which is the key observation to avoid exponential dependence.
\begin{restatable}{theorem}{ThmModelDiffPlus}\label{thm:model_diff_contraction}
    Given two arbitrary model $M = (\cS,\cA,H,\mP_{T},\mP_r)$ and $\tM=(\cS,\cA,H,\mP_{\tT},\mP_r)$, and arbitrary policy $\pi$, under Assump.~\ref{assump:lipschitz} and Assump.~\ref{assump:contraction}, we have:
    \begin{align*}
        \EE_{\pi,M}[\sum_{h=1}^{H} & \|\mP_{T,h}(\cdot|s_h,a_h,\mu^{\pi}_{M,h}) - \mP_{\tT,h}(\cdot|s_h,a_h,\mu^{\pi}_{\tM,h}) \|_\TV] \\
        \leq & (1+\frac{L_T}{1-L_\Gamma})\EE_{\pi,M}[\sum_{h=1}^{H} \|\mP_{T,h}(\cdot|s_h,a_h,\mu^{\pi}_{M,h}) \\
        &\qquad\qquad\qquad\qquad - \mP_{\tT,h}(\cdot|s_h,a_h,\mu^{\pi}_{M,h}) \|_\TV].
    \end{align*}
\end{restatable}

\section{DISCUSSION AND OPEN PROBLEMS}
In this paper, we study the statistical efficiency of function approximation in MFRL. We propose the notion of MF-MBED and an MLE based algorithm, which can guarantee to efficiently solve MFC and MFG given function classes satisfying basic assumptions including realizablility and Lipschitz continuity.
Under additional structural assumptions, the exponential dependence on Lipschitz factors in sample complexity bounds could be avoided.
In the following, we discuss some potentially interesting open problems.

\paragraph{Tighter Sample Complexity Upper Bounds and Lower Bounds}
Although in this paper, we show that both MFC and MFG can be solved under the same MLE-based model learning framework and establish the same sample complexity depending on MF-MBED, it does not implies MFC and MFG fundamentally have the same sample efficiency.
It would be an interesting direction to investigate other complexity measure which may provide tighter complexity upper bounds than our MF-MBED.
Besides the sample complexity upper bound, another interesting question would be the sample complexity lower bound for MFC and MFG, and whether there exists separation between MFC and MFG in terms of sample efficiency.

\paragraph{Computational Efficiency}
Since we focus on the fundamental sample efficiency, we ignore complicated computation processes and abstract them as computational oracles.
For MFC setting, similar maximization oracles has been treated as mild assumptions .
For MFG setting, although previous literature \citep{guo2019learning,perolat2021scaling} has provided concrete implementations given additional structural assumptions like contractivity and monotonicity, whether NE can be solved efficiently under more general setting is an open problem.
Therefore, another important direction is to combine our results with optimization techniques to design computationally efficient algorithms.

\paragraph{Model-Free Methods in Function Approximation Setting}
In this paper, we consider the model-based methods. Given the popularity of model-free methods in single-agent setting, it remains an open problem what the sample efficiency of model-free methods with general function approximations are.
Even though, we prefer model-based function approximation in mean-field setting because model-free methods suffer from some additional challenges seems intractable. 
For pure value-based methods, without explicit model estimation, one may not infer the density, and therefore may not estimate the true value function accurately.
Policy-based methods would be more promising and has been applied in tabular setting \citep{guo2019learning, yardim2022policy}, where by introducing the policy, one just need to focus on the estimation of values conditioning on policy and density estimation will be unimportant. 
But when generalizing to function approximation setting, the value function class should be powerful enough to approximate the value functions regarding all the policies occurred in the learning process, which might be very strong assumptions.
We leave the investigation of this direction to the future.

\subsubsection*{Acknowledgements}
The work is supported by ETH research grant, NCCR Automation, Swiss National Science Foundation (SNSF)  Project Funding No. 200021-207343. and SNSF Starting Grant.
The authors would like to thank Andreas Krause for valuable discussion and thank Lars Lorch for valuable feedbacks on paper writing.

\bibliographystyle{apalike}
\bibliography{references}

\clearpage
\section*{Checklist}

\begin{enumerate}
 \item For all models and algorithms presented, check if you include:
 \begin{enumerate}
   \item A clear description of the mathematical setting, assumptions, algorithm, and/or model. [Yes]
   \item An analysis of the properties and complexity (time, space, sample size) of any algorithm. [Yes]
   \item (Optional) Anonymized source code, with specification of all dependencies, including external libraries. [Not Applicable]
 \end{enumerate}

 \item For any theoretical claim, check if you include:
 \begin{enumerate}
   \item Statements of the full set of assumptions of all theoretical results. [Yes]
   \item Complete proofs of all theoretical results. [Yes]
   \item Clear explanations of any assumptions. [Yes]
 \end{enumerate}

 \item For all figures and tables that present empirical results, check if you include:
 \begin{enumerate}
   \item The code, data, and instructions needed to reproduce the main experimental results (either in the supplemental material or as a URL). [Not Applicable]
   \item All the training details (e.g., data splits, hyperparameters, how they were chosen). [Not Applicable]
         \item A clear definition of the specific measure or statistics and error bars (e.g., with respect to the random seed after running experiments multiple times). [Not Applicable]
         \item A description of the computing infrastructure used. (e.g., type of GPUs, internal cluster, or cloud provider). [Not Applicable]
 \end{enumerate}

 \item If you are using existing assets (e.g., code, data, models) or curating/releasing new assets, check if you include:
 \begin{enumerate}
   \item Citations of the creator If your work uses existing assets. [Not Applicable]
   \item The license information of the assets, if applicable. [Not Applicable]
   \item New assets either in the supplemental material or as a URL, if applicable. [Not Applicable]
   \item Information about consent from data providers/curators. [Not Applicable]
   \item Discussion of sensible content if applicable, e.g., personally identifiable information or offensive content. [Not Applicable]
 \end{enumerate}

 \item If you used crowdsourcing or conducted research with human subjects, check if you include:
 \begin{enumerate}
   \item The full text of instructions given to participants and screenshots. [Not Applicable]
   \item Descriptions of potential participant risks, with links to Institutional Review Board (IRB) approvals if applicable. [Not Applicable]
   \item The estimated hourly wage paid to participants and the total amount spent on participant compensation. [Not Applicable]
 \end{enumerate}

 \end{enumerate}

\newpage
\appendix
\onecolumn
\tableofcontents

\newpage

\section{Extended Introduction}
\subsection{Additional Related Works}\label{appx:additional_related_works}
\paragraph{Single-Agent RL with General Function Approximation}
Recently, beyond tabular RL \citep{auer2008near, azar2017minimax, jin2018q}, there are significant progress on sample complexity analysis in linear function approximation setting \citep{jin2020provably, zanette2020learning, agarwal2020flambe, modi2021model, uehara2021representation,huang2022towards} or more general function approximation settings \citep{russo2013eluder, jiang2017contextual, sun2019model, jin2021bellman, du2021bilinear, xie2022role, foster2021statistical, chen2022general, ayoub2020model}.
However, the MFRL setting is significantly different from single-agent RL because of the dependence on the state density in transition and reward functions. 
The function complexity measure, especially for value-based function class, and the corresponding algorithms in single-agent RL cannot be trivially generalized to MFRL.

\paragraph{Multi-Agent RL}
Sample complexity of learning in Markov Games has been studied extensively in a recent surge of works \citep{jin2021v,bai2020near,chen2022almost,zhang2019policy,zhang2021multi,xiong2022self}.
A few recent works also consider learning Markov Games with linear or general function approximation \citep{xie2020learning, huang2021towards, jin2022power, ni2022representation}.
None of these results can be directly extended to Mean-Field RL.

Recently, \citep{wang2023breaking,cui2023breaking} also studied how to ``break the curse of multi-agency'' by decentralized learning in MARL setting.
Although they consider a more general setting from ours, where agents can be largely different, there are still some restrictions when generalizing their results to our mean-field setting.
First of all, their algorithm can only guarantee the convergence to the Coarse Correlated Equilibria or the Correlated Equilibria, while ours can converge to Nash Equilibrium in MFG.
Moreover, and more importantly, the sample complexity of their algorithms depend on the number of agents (although polynomially instead of exponentially), which implies that their algorithms still suffer from the ``curse of multi-agency'' when the number of agents is exponentially large.

\paragraph{Other Related Works}
In this paper, we consider MLE based model estimation algorithm. Similar ideas has been adopted in POMDP \citep{liu2022partially} or Partial Observable Markov Games \citep{liu2022sample}.
\section{Proofs for Eluder Dimension Related}
\subsection{Missing Details of Eluder Dimension Related}\label{appx:missing_details_MF-MBED}
In the following, we recall the Eluder Dimension in Value Function Approximation Setting \citep{russo2013eluder}.
\begin{definition}[$\epsilon$-Independence for Scalar Function]
    Given a domain $\cY$ and a function class $\cF\subset \{f|f:\cY\rightarrow\mR\}$, we say $y$ is $\epsilon$-independent w.r.t. $y_1,y_2,...,y_n$, if there exists $f_1,f_2\in\cF$ satisfying $\sqrt{\sum_{i=1}^{n} |f_1(y_i) - f_2(y_i)|^2} \leq \epsilon$ but $|f_1(y) - f_2(y)| > \epsilon$.
\end{definition}
\begin{definition}[Eluder Dimension for Scalar Function]
    Given a function class $\cF\subset \{g|g:\cY\rightarrow\mR\}$, we use $\overline{\dimE}(\cF,\epsilon)$ to denote the length of the longest sequence $y_1,...,y_n \in \cY$, such that, for any $i \in [n]$, $y_i$ is $\epsilon$-independent w.r.t. $y_1,...,y_{i-1}$.
\end{definition}
\paragraph{Remarks on Assump.~\ref{assump:lipschitz}}
The main reason we require the Lipschitz continuity w.r.t. Hellinger distance is to handle the distribution shift issue.
In Thm.~\ref{thm:MLE_Gaurantee}, we show that MLE regression can only guarantee the learned model aligns with $M^*$ under different density. In order to guarantee efficient learning, we need to convert it to upper bound for model error under the same density.

Besides, although in general $\mH$ and $\matDTV$ distance between two distribution can be largely different. For our example in Prop.~\ref{prop:MF-MBED_det_trans_noise}, given two function $f_1,f_2$, we have:
\begin{align*}
    \mH(\mP_{f_1},\mP_{f_2}) = O(\frac{1}{\sigma^2}\|f_1 - f_2\|_2) = O(\frac{1}{\sigma^2}\|f_1 - f_2\|_1).
\end{align*}
Therefore, Assump.~\ref{assump:lipschitz} can be ensured when $f\in\cF$ is Lipschitz w.r.t. $l_1$ distance, which is reasonable.

Moreover, in fact, if we only consider $\dimE_\alpha(\cM,\matDTV,\epsilon)$ as model-based eluder dimension in our framework, we only require Lipschitz continuity w.r.t. $l_1$-distance (or $\matDTV$ distance).

\subsection{Concrete Examples Satisfying Finite Eluder Dimension Assumption}\label{appx:examples_MF-MBED}

\subsubsection{Example 1: Linear Combined Model}
\begin{proposition}[Linearly Combined Model]\label{prop:MF-MBED_Linear_Combined_formal}
    Consider the linear combined model class with known state action feature vector $\phi(s,a,\mu,s') \in \mR^{d}$, such that for arbitrary $s\in\cS,a\in\cA$ and arbitrary $g:\cS\rightarrow[0,1]$, we have $\|\SumInt_{s'\in\cS} \phi(s,a,\mu,s') g(s')\|_2 \leq C_\phi$\footnote{Similar normalization assumptions is common in previous literatures \citep{agarwal2020flambe, modi2020sample, uehara2021representation}\label{fn:normalization}}
    \begin{align*}
        \cP := \{\mP_\theta|\mP_\theta(\cdot|s,a,\mu):=\theta\trans \phi(s,a,\mu,s'),~ \|\theta\|_2 \leq C_\theta;~ \forall s,a,\mu,~\SumInt_{s'\in\cS}\mP(s'|s,a,\mu)= 1,~ \mP(\cdot|s,a,\mu) \geq 0\}.
    \end{align*}
    For $\alpha \geq 1$, we have: $\dimE_\alpha(\cP, \matDTV,\epsilon) = O(d\log (1+\frac{dC_\theta C_\phi}{\epsilon}))$.
\end{proposition}
\begin{proof}
    We focus on the case when $\alpha = 1$ since which directly serves as upper bound for $\alpha > 1$.
    For arbitrary $\theta_1,\theta_2$ with $\|\theta_1\|_2\leq C_\theta,\|\theta_2\|_2\leq C_\theta$, we have:
    \begin{align*}
        \matDTV(\mP_{\theta_1},\mP_{\theta_2})(s,a,\mu) =& \sup_{\bcS}|\SumInt_{s'} (\theta_1-\theta_2)\trans \phi(s,a,\mu,s')|=\frac{1}{2}(\theta_1-\theta_2)\trans \SumInt_{s'} \phi(s,a,\mu,s') g_{\theta_1,\theta_2}(s,a,\mu,s').
    \end{align*}
    where we define:
    \begin{align*}
        g_{\theta_1,\theta_2}(s,a,\mu,s'):=\begin{cases}
            1,\quad \text{if~} (\theta_1-\theta_2)\trans \phi(s,a,\mu,s') \geq 0;\\
            -1,\quad \text{otherwise}.
        \end{cases}
    \end{align*}
    In the following, for simplicity, we use 
    $$
    v_{\theta_1,\theta_2}(s,a,\mu):=\SumInt_{s'} \phi(s,a,\mu,s') g_{\theta_1,\theta_2}(s,a,\mu,s').
    $$ 
    as a short note. Also note that, 
    \begin{align*}
        \|v_{\theta_1,\theta_2}(s,a,\mu)\|_2 \leq \|\SumInt_{s'} \phi(s,a,\mu,s') g_{\theta_1,\theta_2}(s,a,\mu,s')\|_2 \leq C_\phi,\quad \forall \pi,\mu,\theta_1,\theta_2 \in \cB(0; C_\theta).
    \end{align*}
    Suppose we have a sequence of samples $x_1,..,x_n$ with $x_i := (s^i,a^i,\mu^i)$, such that $x_i$ is $\epsilon$-independent of $\{x_1,...,x_{i-1}\}$ for all $i\in[n]$. Then, by definition, for each $i$, there exists $\theta_1^i, \theta_2^i$ such that:
    \begin{align*}
        4\epsilon^2 \leq& 4\|\mP_{\theta^i_1}(\cdot|s^i,a^i,\mu^i) - \mP_{\theta^i_2}(\cdot|s^i,a^i,\mu^i)\|_\TV^2\\
        =&\Big((\theta^i_1-\theta^i_2)\trans v_{\theta^i_1,\theta^i_2}(s^i,a^i,\mu^i)\Big)^2\leq  \|\theta_1^i-\theta_2^i\|^2_{\Lambda^i} \|v_{\theta^i_1,\theta^i_2}(s^i,a^i,\mu^i)\|^2_{(\Lambda^i)^{-1}}.
    \end{align*}
    where we denote
    \begin{align*}
        \Lambda^i :=& \lambda I + \sum_{t=1}^{i-1} 
        v_{\theta_1^t,\theta_2^t}(s^t,a^t,\mu^t)\trans v_{\theta_1^t,\theta_2^t}(s^t,a^t,\mu^t).
    \end{align*}
    Meanwhile,
    \begin{align*}
        \|\theta_1^i-\theta_2^i\|^2_{\Lambda^i} =& \lambda \|\theta_1^i-\theta_2^i\|^2 + \sum_{t=1}^{i-1} \Big((\theta_1^i-\theta_2^i)\trans v_{\theta_1^t,\theta_2^t}(s^t,a^t,\mu^t)\Big)^2\\
        \leq & 4\lambda  C_\theta^2 + \sum_{t=1}^{i-1} \Big((\theta^i_1-\theta^i_2)\SumInt_{s'} \Phi(s^t,a^t,\mu^t)\psi(s') g_{\theta_1^t,\theta_2^t}(s^t,a^t,\mu^t,s')\Big)^2\\
        = & 4\lambda C_\theta^2 + 4\sum_{t=1}^{i-1} \|\mP_{\theta_1^t}(\cdot|s^t,a^t,\mu^t) ,\mP_{\theta_2^t}(\cdot|s^t,a^t,\mu^t) \|_\TV\tag{$|g_{\theta_1^t,\theta_2^t}(\cdot,\cdot,\cdot,\cdot)| = 1$}\\
        \leq & 4\lambda C_\theta^2 + 4\epsilon^2.
    \end{align*}
    By choosing $\lambda = \epsilon^2/ C_\theta^2$, we further have:
    \begin{align*}
        \|v_{\theta^i_1,\theta^i_2}(s^i,a^i,\mu^i)\|^2_{(\Lambda^i)^{-1}} \geq \frac{4\epsilon^2}{4 \lambda  C_\theta^2 + 4\epsilon^2} = \frac{1}{2}.
    \end{align*}
    On the one hand,
    \begin{align*}
        \det \Lambda^{n+1} =& \det (\Lambda^n + v_{\theta^n_1,\theta^n_2}(s^n,a^n,\mu^n)v_{\theta^n_1,\theta^n_2}(s^n,a^n,\mu^n)\trans )\\
        = &  (1 + v_{\theta^n_1,\theta^n_2}(s^n,a^n,\mu^n)\trans (\Lambda^n)^{-1} v_{\theta^n_1,\theta^n_2}(s^n,a^n,\mu^n)) \cdot \det \Lambda^n\\
        \geq & \frac{3}{2}\det \Lambda^n \geq (\frac{3}{2})^n \det \Lambda^1 = \lambda^d (\frac{3}{2})^n.
    \end{align*}
    On the other hand,
    \begin{align*}
        \lambda^d (\frac{3}{2})^n  \leq \det \Lambda^{n+1} \leq (\frac{\Tr(\Lambda^n)}{d})^d \leq (\lambda + \frac{n C^2_\phi}{d})^d.
    \end{align*}
    which implies $n = O(d\log (1+\frac{dC_\theta C_\phi}{\epsilon}))$.
\end{proof}

\paragraph{Linear Combined Model with State-Action-Dependent Weights}
In \citep{modi2020sample}, the authors introduced another style of linear combined model with state-action dependent weights, which can be generalized to MFRL setting by:
\begin{align*}
    \mP_W(s'|s,a,\mu) := \sum_{i=1}^d [W\phi(s,a,\mu,s')]_k \mP_i(s'|s,a,\mu).
\end{align*}
where $W \in \mR^{d\times d}$ is an unknown matrix, $\phi(s,a)$ are known feature class, $\{\mP_i\}_{i=1}^d$ are $d$ known models to combine.
If we further define $\psi(s,a,\mu,s') := [\mP_1(s'|s,a,\mu),...,\mP_d(s'|s,a,\mu)]\trans \in \mR^{d}$, we can rewrite the model by:
\begin{align*}
    \mP_W(s'|s,a,\mu) = \phi(s,a,\mu,s')\trans W\trans \psi(s,a,\mu,s') = \vectorize(W\trans)\trans \vectorize(\psi(s,a,\mu,s')\phi(s,a,\mu,s')\trans).
\end{align*}
Therefore, by treating $\theta = \vectorize(W\trans)$ to be the parameter and $\vectorize(\psi(s,a,\mu,s')\phi(s,a,\mu,s')\trans)$ to be the feature taking place the role of $\phi(s,a,\mu,s')$ in Prop.~\ref{prop:MF-MBED_Linear_Combined_formal}, we can absorb this model class into linearly combined model framework, and $\tilde{O}(d^2)$ will be an upper bound for its MF-MBED.

\subsubsection{Example 2: Linear MDP with Known Feature}
\begin{proposition}[Low-Rank MF-MDP; Formal Version of Prop.~\ref{prop:MF-MBED_Linear_MFMDP}]\label{prop:MF-MBED_Linear_MFMDP_formal}
    Consider the Low-Rank MF-MDP with known feature $\phi:\cS\times\cA\times\Delta(\cS)\rightarrow \mR^d$ satisfying $\|\phi\| \leq C_\phi$, and unknown next state feature $\psi:\cS\rightarrow\mR^d$. Given a next state feature function class $\Psi$ satisfying $\forall \psi\in\Psi,~\forall s'\in\cS,~\forall g:\cS\rightarrow \{-1,1\}$, $\|\SumInt_{s'}\psi(s')g(s')\|_2\leq C_\Psi$, consider the following model class:
    \begin{align*}
        \cP_\Psi := \{\mP_\psi|\mP_\psi(\cdot|s,a,\mu):=\phi(s,a,\mu)\trans\psi(s');\forall s,a,\mu,~\SumInt_{s'\in\cS}\mP_\psi(s'|s,a,\mu) = 1,~ \mP_\psi(s'|s,a,\mu) \geq 0;\psi\in \Psi\},
    \end{align*}
    for $\alpha \geq 1$, we have $\dimE_\alpha(\cP_\Psi,\matDTV, \epsilon) = O(d\log (1+\frac{dC_\phi C_\Psi}{\epsilon}))$.
\end{proposition}
\begin{proof}
    Again we focus on the case when $\alpha = 1$.
    Suppose there is a sequence of samples $x_1,...,x_n$ (with $x_i := (s^i,a^i,\mu^i)$) such that for any $i\in[n]$, $x_i$ is $\epsilon$-independent w.r.t. $x_1,...,x_{i-1}$ w.r.t. $\cP_\Psi$ and $\matDTV$, then for each $i\in[n]$, there should exists $\psi^i,\tpsi^i \in \Psi$, such that:
    \begin{align*}
        \epsilon^2 \geq \sum_{t=1}^{i-1} \|\mP_{\psi^i}(\cdot|s^t,a^t,\mu^t), \mP_{\tpsi^i}(\cdot|s^t,a^t,\mu^t)\|_\TV^2.
    \end{align*}
    and
    \begin{align*}
        \epsilon^2 \leq & \|\mP_{\psi^i}(\cdot|s^i,a^i,\mu^i) - \mP_{\tpsi^i}(\cdot|s^i,a^i,\mu^i)\|_\TV^2\\
        =&\sup_{\bcS\subset\cS}\Big(\SumInt_{s'\in\bcS}\phi(s^i,a^i,\mu^i)\trans({\psi^i}(s')-\tpsi^i(s'))\Big)^2\\
        =&\frac{1}{4}\Big(\phi(s^i,a^i,\mu^i)\trans\SumInt_{s'\in\cS}({\psi^i}(s')-\tpsi^i(s'))g_{{\psi^i},\tpsi^i}(s^i,a^i,\mu^i,s')\Big)^2\\
        \leq & \frac{1}{4}\|\phi(s^i,a^i,\mu^i)\|_{(\Lambda^i)^{-1}}^2\|\SumInt_{s'\in\cS}({\psi^i}(s')-\tpsi^i(s'))g_{{\psi^i},\tpsi^i}(s^i,a^i,\mu^i,s')\|_{\Lambda^i}^2.
    \end{align*}
    where we define:
    \begin{align*}
        \Lambda^i := \lambda I + \sum_{t=1}^{i-1} \phi(s^i,a^i,\mu^i)\phi(s^i,a^i,\mu^i)\trans;\quad 
        g_{\psi^i,\tpsi^i}(s,a,\mu,s'):=&\begin{cases}
            1,\quad \text{if~} \phi(s^i,a^i,\mu^i)\trans({\psi^i}(s')-\tpsi^i(s')) \geq 0;\\
            -1,\quad \text{otherwise}.
        \end{cases}
    \end{align*}
    For simplicity, we use $v_{\psi,\tpsi}(s,a,\mu) := \SumInt_{s'}({\psi}(s')-\tpsi(s'))g_{{\psi},\tpsi}(s,a,\mu,s')$ as a shortnote. Therefore, for each $i$,
    \begin{align*}
        \|v_{\psi^i,\tpsi^i}(s^i,a^i,\mu^i)\|_{\Lambda^i}^2 =&\lambda \|v_{\psi^i,\tpsi^i}(s^i,a^i,\mu^i)\|^2 + \sum_{t=1}^{i-1}\Big(\phi(s^t,a^t,\mu^t)\trans v_{\psi^i,\tpsi^i}(s^i,a^i,\mu^i)\Big)^2\\
        =&\lambda \|v_{\psi^i,\tpsi^i}(s^i,a^i,\mu^i)\|^2 + \sum_{t=1}^{i-1}\Big(\phi(s^t,a^t,\mu^t)\trans\SumInt_{s'} (\psi^i(s)-\tpsi^i(s'))g_{\psi^i,\tpsi^i}(s^i,a^i,\mu^i,s')\Big)^2\\
        =& 4\lambda C_\Psi^2 + 4\sum_{t=1}^{i-1} \|\mP_{\psi^i}(\cdot|s^t,a^t,\mu^t) - \mP_{\tpsi^i}(\cdot|s^t,a^t,\mu^t)\|_\TV^2\\
        \leq & 4\lambda C_\Psi^2 + 4\epsilon^2.
    \end{align*}
    By choosing $\lambda = \epsilon^2/C^2_\Psi$, we have:
    \begin{align*}
        \|\phi(s^i,a^i,\mu^i)\|_{(\Lambda^i)^{-1}}^2 \geq \frac{4\epsilon^2}{4\lambda C^2_\Psi + 4\epsilon^2} = \frac{1}{2}.
    \end{align*}
    On the one hand,
    \begin{align*}
        \det \Lambda^{n+1} =& \det (\Lambda^n + \phi(s^n,a^n,\mu^n)\phi(s^n,a^n,\mu^n)\trans ) = (1 + \|\phi(s^n,a^n,\mu^n)\|_{(\Lambda^n)^{-1}}^2) \cdot \det \Lambda^n\\
        \geq & \frac{3}{2}\det \Lambda^n\geq (\frac{3}{2})^n \det \Lambda^1 = \lambda^d (\frac{3}{2})^n.
    \end{align*}
    Therefore,
    \begin{align*}
        \lambda^d (\frac{3}{2})^n \leq \det \Lambda^{n+1} \leq (\frac{\Tr(\Lambda^n)}{d})^d \leq (\lambda + \frac{n C^2_\phi}{d})^d.
    \end{align*}
    which implies $n = O(d\log(1+\frac{d C_\phi C_\Psi}{\epsilon}))$.
\end{proof}

\subsubsection{Example 3: Kernel Mean-Field MDP}
We first introduce the notion of Effective Dimension, which is also known as the critical information gain in \citep{du2021bilinear}:
\begin{definition}[Effective Dimension]
    The $\epsilon$-effective dimension of a set $\cY$ is the minimum integer $d_\eff(\cY,\epsilon) = n$, such that,
    \begin{align*}
        \sup_{y_1,...,y_n\in \cY} \frac{1}{n} \log\det (I + \frac{1}{\epsilon^2}\sum_{i=1}^n y_i y_i\trans) \leq \frac{1}{e}.
    \end{align*}
\end{definition}
In the next theorem, we show that, the MF-MBED of kernel MF-MDP generalized from kernel MDP in single-agent setting \citep{jin2021bellman} can be upper bounded by the effiective dimension in certain Hilbert spaces.
\begin{proposition}[Kernel MF-MDP]
    Given a separable Hilbert space $\cH$, a feature mapping $\phi:\cS\times\cA\times\Delta(\cS) \rightarrow \cH$ such that $\|\phi(s,a,\mu)\|_\cH \leq C_\phi$ for all $s\in\cS,a\in\cA,\mu\in\Delta(\cS)$, and a next state feature class $\Psi\subset\{\psi:\cS \rightarrow \cH\}$ satisfying the normalization property that $\forall~\psi\in\Psi$ and $g:\cS\rightarrow\{-1,1\}$, $\|\SumInt_{s'\in\cS}\psi(s')g(s')\|_\cH \leq 1$ \footnote{To align with \citep{jin2021bellman}, we assume $\psi$ is normalized.}.
    Consider the model class $\cP_{\Psi, \cH}$ defined by:
    \begin{align*}
        \cP_{\Psi,\cH} := \{\mP_\psi|\mP_\psi(s'|s,a,\mu) = \langle \phi(s,a,\mu), \psi(s') \rangle_\cH,~\SumInt_{s'\in\cS}\mP_\psi(s'|s,a,\mu) = 1,~\mP_\psi(\cdot|s,a,\mu) \geq 0,~\psi \in \Psi\}.
    \end{align*}
    For $\alpha \geq 1$, we have
    \begin{align*}
        \dimE_\alpha(\cP_{\Psi,\cH},\matDTV,\epsilon) = O(d_\eff(\phi(\cX),\epsilon)),
    \end{align*}
    where we use $\cX:=\cS\times\cA\times\Delta(\cS)$ as a short note, and $\phi(\cX) := \{\phi(x)|x\in\cX\}$.
\end{proposition}
\begin{proof}
    The proof idea is similar to the proof of Prop.~\ref{prop:MF-MBED_Linear_MFMDP_formal}. Again, we only focus on the case when $\alpha = 1$.
    Suppose there is a sequence of samples $x_1,...,x_n$ (with $x_i := (s^i,a^i,\mu^i)$) such that for any $i\in[n]$, $x_i$ is $\epsilon$-independent w.r.t. $x_1,...,x_{i-1}$ w.r.t. $\cP_{\Psi,\cH}$ and $\matDTV$, then for each $i\in[n]$, there should exists $\psi^i,\tpsi^i \in \Psi$, such that:
    \begin{align*}
        \epsilon^2 \geq \sum_{t=1}^{i-1} \|\mP_{\psi^i}(\cdot|s^t,a^t,\mu^t) - \mP_{\tpsi^i}(\cdot|s^t,a^t,\mu^t)\|_\TV^2.
    \end{align*}
    and
    \begin{align*}
        4\epsilon^2 \leq & 4\|\mP_{\psi^i}(\cdot|s^i,a^i,\mu^i) - \mP_{\tpsi^i}(\cdot|s^i,a^i,\mu^i)\|_\TV^2\\
        =&\Big(\langle \phi(s^i,a^i,\mu^i),\SumInt_{s'}({\psi^i}(s')-\tpsi^i(s'))g_{{\psi^i},\tpsi^i}(s^i,a^i,\mu^i,s') \rangle_\cH \Big)^2\\
        \leq & \|\phi(s^i,a^i,\mu^i)\|_{(\Lambda^i)^{-1}}^2\|\SumInt_{s'}({\psi^i}(s')-\tpsi^i(s'))g_{{\psi^i},\tpsi^i}(s^i,a^i,\mu^i,s')\|_{\Lambda^i}^2.
    \end{align*}
    where we define:
    \begin{align*}
        \Lambda^i := \lambda I + \sum_{t=1}^{i-1} \phi(s^i,a^i,\mu^i)\phi(s^i,a^i,\mu^i)\trans;\quad 
        g_{\psi^i,\tpsi^i}(s,a,\mu,s'):=&\begin{cases}
            1,\quad \text{if~} \langle \phi(s^i,a^i,\mu^i), {\psi^i}(s')-\tpsi^i(s')\rangle_\cH \geq 0;\\
            -1,\quad \text{otherwise}.
        \end{cases}
    \end{align*}
    Based on a similar discussion and choice of $\lambda = \epsilon^2$, as Prop.~\ref{prop:MF-MBED_Linear_MFMDP_formal}, we have:
    \begin{align*}
        (\frac{3}{2})^n \det \Lambda^1 \leq \det \Lambda^{n+1} = \det (\epsilon^2 I + \sum_{i=1}^n \phi(s^i,a^i,\mu^i) \phi(s^i,a^i,\mu^i)\trans),
    \end{align*}
    Therefore,
    \begin{align*}
        n \log \frac{3}{2} \leq \frac{\det \Lambda^{n+1}}{\det \Lambda^1} = \det (I + \frac{1}{\epsilon^2} \sum_{i=1}^n \phi(s^i,a^i,\mu^i) \phi(s^i,a^i,\mu^i)\trans) \leq \frac{1}{e} d_\eff(\phi(\cX),\epsilon),
    \end{align*}
    which implies $n = O(d_\eff(\phi(\cX),\epsilon))$.
\end{proof}

\subsubsection{Example 4: Generalized Linear Function Class}
In this section, we extend the Generalized Linear Models in single-agent RL \citep{russo2013eluder} to MF-MDP.
\begin{proposition}[Generalized Linear MF-MDP]
Given a differentiable and strictly increasing function $h:\mR\rightarrow\mR$ satisfying $ 0 < \underline{h}\leq h' \leq \overline{h}$, where $h'$ is its derivative, suppose we have a feature mapping $\phi:\cS\times\cA\times\Delta(\cS)\rightarrow \mR^d$ satisfying $\|\phi(\cdot,\cdot,\cdot)\|_2 \leq C_\phi$ and a feature class $\Psi\subset \{\psi|\psi:\cS\rightarrow\mR\}$ such that for any $\psi \in \Psi$, $\|\SumInt_{s'\in\cS} \psi(s')g(s')\|_2 \leq C_\Psi$ for any $g:\cS\rightarrow\{-1,1\}$. 
Consider the model class:
\begin{align*}
    \cP_{h,\Psi} := \{\mP_\psi|\mP_\psi(\cdot|s,a,\mu):=h(\phi(s,a,\mu)\trans\psi(s'));\forall s,a,\mu,~\|\mP_\psi(\cdot|s,a,\mu)\|_1 = 1,~ \mP_\psi(s'|s,a,\mu) \geq 0;\psi\in \Psi\},
\end{align*}
For any $\alpha \geq 1$, we have $\dimE_\alpha(\cP_{h,\Psi}, \matDTV, \epsilon) = \tilde{O}(dr^2)$, where $r := \overline{h}/\underline{h}$.
\end{proposition}
\begin{proof}
    The proof is similar to Prop.~\ref{prop:MF-MBED_Linear_MFMDP_formal}.
    Suppose there is a sequence of samples $x_1,...,x_n$ (with $x_i := (s^i,a^i,\mu^i)$) such that for any $i\in[n]$, $x_i$ is $\epsilon$-independent w.r.t. $x_1,...,x_{i-1}$ w.r.t. $\cP_\Psi$ and $\matDTV$, then for each $i\in[n]$, there should exists $\psi^i,\tpsi^i \in \Psi$, such that:
    \begin{align*}
        \epsilon^2 \geq & \sum_{t=1}^{i-1} \|\mP_{\psi^i}(\cdot|s^t,a^t,\mu^t) - \mP_{\tpsi^i}(\cdot|s^t,a^t,\mu^t)\|_\TV^2 \\
        = & \sum_{t=1}^{i-1} \sup_{\bcS\subset\cS}\Big(\SumInt_{s'\in\bcS} h(\phi(s^t,a^t,\mu^t)\trans \psi^i(s')) - h(\phi(s^t,a^t,\mu^t)\trans \tpsi^i(s'))\Big)^2 \\
        \geq & \underline{h}^2 \sum_{t=1}^{i-1} \sup_{\bcS\subset\cS}\Big(\SumInt_{s'\in\bcS} \phi(s^t,a^t,\mu^t)\trans (\psi^i(s')-\tpsi^i(s'))\Big)^2 \tag{Mean Value Theorem}.
    \end{align*}
    Besides,
    \begin{align*}
        4\epsilon^2 \leq &4 \|\mP_{\psi^i}(\cdot|s^i,a^i,\mu^i) - \mP_{\tpsi^i}(\cdot|s^i,a^i,\mu^i)\|_\TV^2\\
        =& 4\sup_{\bcS\subset\cS} \Big(\SumInt_{s'\in\bcS} h(\phi(s^i,a^i,\mu^i)\trans \psi^i(s')) - h(\phi(s^i,a^i,\mu^i)\trans \tpsi^i(s'))\Big)^2\\
        \leq &4 \overline{h}^2 \sup_{\bcS\subset\cS} \Big(\SumInt_{s'\in\bcS}\phi(s^i,a^i,\mu^i)\trans (\psi^i(s')-\tpsi^i(s'))\Big)^2\\
        = & \overline{h}^2 \Big(\SumInt_{s'\in\bcS}\phi(s^i,a^i,\mu^i)\trans (\psi^i(s')-\tpsi^i(s'))g_{{\psi^i},\tpsi^i}(s^i,a^i,\mu^i,s')\Big)^2 \\
        \leq & \overline{h}^2\|\phi(s^i,a^i,\mu^i)\|_{(\Lambda^i)^{-1}}^2\|\SumInt_{s'}({\psi^i}(s')-\tpsi^i(s'))g_{{\psi^i},\tpsi^i}(s^i,a^i,\mu^i,s')\|_{\Lambda^i}^2.
    \end{align*}
    where in the second inequality, we use the mean value theorem and the fact that $h' \leq \overline{h}$; $\Lambda^i$ and $g_{\psi^i,\tpsi^i}$ are the same as those in Prop.~\ref{prop:MF-MBED_Linear_MFMDP_formal}.
    By denoting $v_{\psi,\tpsi}(s,a,\mu) := \SumInt_{s'}({\psi}(s')-\tpsi(s'))g_{{\psi},\tpsi}(s,a,\mu,s')$, similar to the proof in Prop.~\ref{prop:MF-MBED_Linear_MFMDP_formal}, we have the following upper bound:
    \begin{align*}
        &\|v_{\psi^i,\tpsi^i}(s^i,a^i,\mu^i)\|_{\Lambda^i}^2\leq 4\lambda C_\Psi^2 + 4\epsilon^2 / \underline{h}^2.
    \end{align*}
    By choosing $\lambda = \epsilon^2/\underline{h}^2C^2_\Psi$, we have:
    \begin{align*}
        \|\phi(s^i,a^i,\mu^i)\|_{(\Lambda^i)^{-1}}^2 \geq \frac{4\epsilon^2}{\overline{h}(4\lambda C^2_\Psi + 4\epsilon^2/\underline{h}^2)} = \frac{1}{r^2}.
    \end{align*}
    By a similar discussion, we have:
    \begin{align*}
        (1+\frac{1}{r^2})^n \det \Lambda^1 \leq \det \Lambda^{n+1} \leq (\lambda + \frac{nC_\phi^2}{d})^d.
    \end{align*}
    which implies:
    \begin{align*}
        n = O(d\log(1 + \frac{\overline{h} dC_\phi C_\Psi}{\epsilon})/ \log (1+ \frac{1}{r^2})) = O(dr^2 \log(1 + \frac{\overline{h} dC_\phi C_\Psi}{\epsilon})).
    \end{align*}
\end{proof}

\subsubsection{Example 5: Deterministic Transition with Gaussian Noise}
\PropExampleHD*
\begin{proof}
    First of all, consider the function $h(x) = 1 - \exp(-x/8)$, in general, we have:
    \begin{align*}
        \frac{x}{8} \geq h(x).
    \end{align*}
    Besides, for $x \in [0, 1]$, we have $0\leq h(x)\leq 1 - \exp(-1/8)$ and
    \begin{align*}
        h(x) = 1 - \exp(-\frac{x}{8}) = \exp(0) - \exp(-\frac{x}{8}) \geq -\frac{\exp(-\frac{1}{8}) - \exp(0)}{1 - 0} x > \frac{1}{16}x.
    \end{align*}
    Given $\epsilon \leq 0.3 < \sqrt{1 - \exp(-1/8)}$, suppose we have a sequence of samples $x_1,...,x_{n}\in\cX:=\cS\times\cA\times\Delta(\cS)$, with $x_i := (s^i,a^i,\mu^i)$ , such that for any $i\in[n]$, $x_i$ is $\alpha$-weakly-$\epsilon$-independent w.r.t. $x_1,...,x_{i-1}$. For any $i\in[n]$, there must exists $g_i^1, g_i^2 \in \cG$ such that, $f_{g_i^1},f_{g_i^2} \in \cF_\cG$, and
    \begin{align*}
        \epsilon^2 \geq & \sum_{j=1}^{i-1} \mH^2(\mP_{f_{g_i^1}}(\cdot|s^j,a^j,\mu^j), \mP_{f_{g_i^2}}(\cdot|s^j,a^j,\mu^j)) \\
        =& \sum_{j=1}^{i-1} h(\|f_{g_i^1}(s^j,a^j,\mu^j) - f_{g_i^2}(s^j,a^j,\mu^j)\|_{\Sigma^{-1}}^2) \\
        \geq & \sum_{j=1}^{i-1} \frac{1}{16\sigma^2}\|f_{g_i^1}(s^j,a^j,\mu^j) - f_{g_i^2}(s^j,a^j,\mu^j)\|_2^2.\\
        =& \frac{1}{16\sigma^2} \sum_{j=1}^{i-1} \sum_{t=1}^d |g_i^1(s^j,a^j,\mu^j,t) - g_i^2(s^j,a^j,\mu^j,t)|^2.
    \end{align*}
    and
    \begin{align*}
        \alpha^2\epsilon^2 <& \mH^2(\mP_{f_{g_i^1}}(\cdot|s^i,a^i,\mu^i), \mP_{f_{g_i^2}}(\cdot|s^i,a^i,\mu^i)) \\
        \leq & \frac{1}{8\sigma^2}\|f_{g_i^1}(s^i,a^i,\mu^i) - f_{g_i^2}(s^i,a^i,\mu^i)\|_2^2\\
        =& \frac{1}{8\sigma^2} \sum_{t=1}^d |g_i^1(s^i,a^i,\mu^i,t) - g_i^2(s^i,a^i,\mu^i,t)|^2\\
        \leq &  \frac{d}{8\sigma^2} \max_{t\in[d]}|g_i^1(s^i,a^i,\mu^i,t) - g_i^2(s^i,a^i,\mu^i,t)|^2.
    \end{align*}
    By choosing $\alpha = \sqrt{2}$, we know that, for any $i\in[n]$, $x_i$ is $4\sigma\epsilon$-independent w.r.t. $\{x_1,...,x_{i-1}\}$ on function class $\cF_\cG$. Therefore,
    \begin{align*}
        \dimE_{\alpha=\sqrt{2}}(\cP_\cG,\mH,\epsilon) \leq \dimE_{\alpha=1}(\cF_\cG,4\sigma\epsilon).
    \end{align*}
    Besides, considering the sequence $t_1,t_2,...,t_n$ with 
    $$
    t_i := \argmax_{t\in[d]}|g_i^1(s^i,a^i,\mu^i,t) - g_i^2(s^i,a^i,\mu^i,t)|^2,
    $$
    and choosing $\alpha = \sqrt{2d}$, we have $(s^{i},a^{i},\mu^{i},t_{i})$ is $4\sigma\epsilon$-independent w.r.t. $\{(s^1,a^1,\mu^1,t_1),...,(s^{i-1},a^{i-1},\mu^{i-1},t_{i-1})\}$ for any $i\in[n]$. Therefore,
    \begin{align*}
        \dimE_{\alpha=\sqrt{2d}}(\cP_\cG,\mH,\epsilon) \leq \overline{\dimE}(\cG,4\sigma\epsilon).
    \end{align*}
\end{proof}

\subsection{From Eluder Dimension to Regret Bound}\label{appx:ED_2_Regret}
\begin{lemma}\label{lem:finite_violation}
    Under the condition and notation as Def.~\ref{def:eps_independent}, consider a fixed $f^* \in \cF$, and suppose we have a sequence $\{f_k\}_{k=1}^K \in \cF$ and $\{x_k\}_{k=1}^K$ with $x_k := (s^k,a^k,\mu^k) \in \cS\times\cA\times\Delta(\cS)$ satisfying that, for all $k\in[K]$, $\sum_{i=1}^{k-1} \matD^2(f_k, f^*)(x_i) \leq \beta$. Then for all $k\in[K]$, and arbitrary $\epsilon > 0$, we have:
    \begin{align*}
        \sum_{k=1}^K \mathbb{I}[\matD(f_k, f^*)(x_k) > \alpha\epsilon] \leq (\frac{\beta}{\epsilon^2} + 1)\dimE_\alpha(\cF,\epsilon).
    \end{align*}
\end{lemma}
\begin{proof}
    We first show that, for some $k$, if $\matD(f_k, f^*)(x_k) > \alpha\epsilon$, then $x_k$ is $\epsilon$-dependent on at most $\beta/\epsilon^2$ disjoint sub-sequence in $\{x_1,...,x_{k-1}\}$. To see this, by Def.~\ref{def:Eluder_Dim}, if $\matD(f_k, f^*)(x_k) > \alpha\epsilon$ and $x_k$ is $\alpha$-weakly-$\epsilon$-dependent w.r.t. a sub-sequence $\{x_{k_1},...,x_{k_\kappa}\} \subset \{x_i\}_{i=1}^{k-1}$, we must have:
    \begin{align*}
        \sum_{i=1}^\kappa \matD^2(f_k,f^*)(x_{k_i}) \geq \epsilon^2.
    \end{align*}
    Given that $\sum_{i=1}^{k-1} \matD^2(f_k,f^*)(x_i) \leq \beta$, the number of such kind of disjoint sub-sequence is upper bounded by $\beta / \epsilon^2$.

    On the other hand, for arbitrary sub-sequence $\{x_{k_1},...,x_{k_\kappa}\} \subset \{x_i\}_{i=1}^{k-1}$, there exists $j\in[\kappa]$ such that $x_{k_j}$ is $\alpha$-weakly-$\epsilon$-dependent on $L:=\lfloor \kappa / \dimE_\alpha(\cF,\epsilon)\rfloor$ disjoint sub-sequence of $\{x_{k_1},...x_{k_{j-1}}\}$. To see this, we first construct $L$ bins $B_1=\{x_{k_1}\},...,B_{L}=\{x_{k_L}\}$. Then, we start with $j =L+1$, and if $x_{k_j}$ is already $\alpha$-weakly-$\epsilon$-dependent w.r.t. sequences $B_1,...,B_L$, then we finish directly. Otherwise, there must exists $B_l$ for some $l\in[L]$ such that $x_{k_j}$ is $\alpha$-weakly-$\epsilon$-independent w.r.t. $B_l$, and we set $B_l \gets B_l \cup \{x_{k_j}\}$ and $j \gets j + 1$. Because the MF-MBED is bounded, $B_l$ can not be larger than $\dimE_\alpha(\cF,\epsilon)$ if the above process continues. Therefore, the process must stop before $j \leq L \cdot \dimE_\alpha(\cF,\epsilon)  \leq \kappa$.

    For arbitrary fixed $k \in [K]$, we use $\{x_{k_1},...,x_{k_\kappa}\} \subset \{x_1,...,x_{k-1}\}$ to denote the elements such that $\matD(f_i,f^*)(x_{k_i}) > \alpha\epsilon$ for $i\in[\kappa]$. There must exists $j\in[\kappa]$, such that, on the one hand, $x_{k_j}$ is $\alpha$-weakly-$\epsilon$-dependent with at most $\beta/\epsilon^2$ disjoint sub-sequence of $\{x_{k_1},...x_{k_{j-1}}\}$, and on the other hand, $x_{k_j}$ is $\alpha$-weakly-$\epsilon$-dependent on at least $L:=\lfloor \kappa / \dimE_\alpha(\cF,\epsilon)\rfloor$ disjoint sub-sequence of $\{x_{k_1},...x_{k_{j-1}}\}$. Therefore, we have:
    \begin{align*}
        \frac{\beta}{\epsilon^2} \geq \lfloor \kappa / \dimE_\alpha(\cF,\epsilon)\rfloor \geq  \kappa  / \dimE_\alpha(\cF,\epsilon) - 1.
    \end{align*}
    which implies $\kappa \leq (\frac{\beta}{\epsilon^2} + 1)\dimE_\alpha(\cF,\epsilon)$.
\end{proof}

\LemEDtoReg*
\begin{proof}
    We first sort the sequence $\{\matD(f_k,f^*)(x_k)\}_{k=1}^K$ and denote them by $e_1,e_2,...,e_k$ with $e_1 \geq e_2...\geq e_K$.  
    For $t\in[K]$, given any $\epsilon > 0$, by Lem.~\ref{lem:finite_violation}, for those $e_t > \alpha\epsilon$, we should have:
    \begin{align*}
        t \leq \sum_{k=1}^K \mathbb{I}[e_k \geq e_t] \leq (\frac{\beta}{e_t^2} + 1)\dimE_\alpha(\cF,\epsilon).
    \end{align*}
    which implies $e_t \leq \sqrt{\frac{\beta\dimE_\alpha(\cF,\epsilon)}{t-\dimE_\alpha(\cF,\epsilon)}}$. Therefore, for any $\epsilon$, we have:
    \begin{align*}
        \sum_{k=1}^K e_k \leq & \alpha K \epsilon + \sum_{k=1}^K \mathbb{I}[e_k > \alpha\epsilon]e_k\\
        \leq & \alpha K \epsilon + (\dimE_\alpha(\cF,\epsilon) + 1) C + \sum_{k=\dimE_\alpha(\cF,\epsilon) + 2}^K \sqrt{\frac{\beta \dimE_\alpha(\cF,\epsilon)}{t - \dimE_\alpha(\cF,\epsilon)}} \tag{Recall the constant $C$ is the upper bound for $\matD(f,f^*)(x)$}\\
        \leq & \alpha K \epsilon + (\dimE_\alpha(\cF,\epsilon) + 1) C + \sqrt{\beta \dimE_\alpha(\cF,\epsilon)} \SumInt_{t=\dimE_\alpha(\cF,\epsilon) + 1}^K \frac{1}{\sqrt{t-\dimE_\alpha(\cF,\epsilon)}} dt\\
        =& O(\sqrt{\beta K \dimE_\alpha(\cF,\epsilon)} + \alpha K \epsilon).
    \end{align*}
\end{proof}

\section{Proofs for MLE Arguments}\label{appx:proof_for_MLE}
In this section, we only provide the proof for the MLE arguments of the algorithm flow for Mean Field Game, where in each iteration, we collect two data w.r.t. two policies in two modes.
One can easily obtain the proof for the DCP of MFC by directly assigning $\tpi = \pi$ and removing the discussion for data $\{\ts,\ta,\ts'\}$, so we omit it.

In the following, given the data collected at iteration $k$, $\cZ^k := \{\{s^k_h,a^k_h,s'^k_{h+1}\}_{h=1}^H \cup \{\ts^k_h,\ta^k_h,\ts'^k_{h+1}\}_{h=1}^H\}$, we use $f^{\pi^k,\tpi^k}_M(\cZ^k)$ to denote the conditional probability w.r.t. model $M$ with transition function $\{\mP_{T,h}\}_{h=1}^H$, i.e.:
\begin{align*}
    f^{\pi^k,\tpi^k}_M(\cZ^k) = \prod_{h\in[H]} \mP_{T,h}(s'^k_{h+1}|s^k_h,a^k_h,\mu^{\pi^k}_{M,h})\mP_{T,h}(\ts'^k_{h+1}|\ts^k_h,\ta^k_h,\mu^{\pi^k}_{M,h}).
\end{align*}
For the simplicity of notations, we divide the random variables in $\cZ^k$ into two parts depending on whether they are conditioned or not:
\begin{align*}
    \cZ^k_{cond} := \{(s_h^k,a_h^k)_{h=1}^H \cup (\ts_h^k,\ta_h^k)_{h=1}^H\},\quad \cZ^k_{pred} := \{(s'^k_{h+1})_{h=1}^H \cup (\ts'^k_{h+1})_{h=1}^H\}.
\end{align*}
Note that for different $h\in[H]$, $(s^k_h,a^k_h,s'^k_{h+1})$ or $(\ts^k_h,\ta^k_h,\ts'^k_{h+1})$ are sampled from different trajectories. Therefore, there is no correlation between $s^k_h,a^k_h$ (or $\ts^k_h,\ta^k_h$) with $s'^k_\ph,a'^k_\ph$ (or $\ts'^k_\ph,\ta'^k_\ph$) for those $h\neq \ph$.
\begin{lemma}\label{lem:MLE_Diff}
    In the following, for the data $\cZ^1,...,\cZ^k$ collected in Alg.~\ref{alg:MLE_MB_Alg} in $M^*$, for any $\delta\in(0,1)$:
    \begin{align*}
        \Pr(\max_{M\in\cM} \sum_{i=1}^k\log \frac{f^{\pi^i,\tpi^i}_M(\cZ^i)}{f^{\pi^i,\tpi^i}_{M^*}(\cZ^i)} \geq \log \frac{|\cM|K}{\delta}) \leq \delta,\quad \forall k \in [K].
    \end{align*}
\end{lemma}
\begin{proof}
    We denote $\EE_k := \EE[\cdot|\{(\pi^i,\tpi^i,\cZ^i)\}_{i=1}^{k-1}\cup \{\pi^k,\tpi^k\}, M^*]$.
First of all, for any $M \in \cM$, we have:
\begin{align*}
    \EE[\exp(\sum_{i=1}^k \log \frac{f^{\pi^i,\tpi^i}_M(\cZ^i)}{f^{\pi^i,\tpi^i}_{M^*}(\cZ^i)})]=&\EE[\exp(\sum_{i=1}^{k-1} \log \frac{f^{\pi^i,\tpi^i}_M(\cZ^i)}{f^{\pi^i,\tpi^i}_{M^*}(\cZ^i)})\EE_k[\exp(\log \frac{f^{\pi^k,\tpi^k}_{M}(\cZ^k)}{f^{\pi^k,\tpi^k}_{M^*}(\cZ^k)})]]\\
    =&\EE[\exp(\sum_{i=1}^{k-1} \log \frac{f^{\pi^i,\tpi^i}_M(\cZ^i)}{f^{\pi^i,\tpi^i}_{M^*}(\cZ^i)})\EE_k[\frac{f^{\pi^k,\tpi^k}_{M}(\cZ^k)}{f^{\pi^k,\tpi^k}_{M^*}(\cZ^k)}]]\\
    =& \EE[\exp(\sum_{i=1}^{k-1} \log \frac{f^{\pi^i,\tpi^i}_M(\cZ^i)}{f^{\pi^i,\tpi^i}_{M^*}(\cZ^i)})]\\
    =& 1.
\end{align*}
Here the last but two step is because:
\begin{align*}
    \EE_k[\frac{f^{\pi^k,\tpi^k}_{M}(\cZ^k)}{f^{\pi^k,\tpi^k}_{M^*}(\cZ^k)}] =& \EE_{\cZ_{cond}^k}[\EE_{\cZ_{pred}^k}[\frac{f^{\pi^k,\tpi^k}_{M}(\cZ^k)}{f^{\pi^k,\tpi^k}_{M^*}(\cZ^k)}|\cZ_{cond}^k, \vecmu^{\pi^k}_{M^*},M^*]|\pi^k,\tpi^k, M^*] \\
    =& \EE_{\cZ_{cond}^k}[\sum_{\cZ^k_{pred}} f^{\pi^k,\tpi^k}_{M^*}(\cZ^k) \frac{f^{\pi^k,\tpi^k}_{M}(\cZ^k)}{f^{\pi^k,\tpi^k}_{M^*}(\cZ^k)}|\pi^k,\tpi^k, M^*]\\
    =&\EE_{\cZ_{cond}^k}[\sum_{\cZ_{pred}} f^{\pi^k,\tpi^k}_{M}(\cZ^k)|\pi^k,\tpi^k, M^*] = \EE_{\cZ^k_{cond}}[1||\pi^k,\tpi^k, M^*]=1.
\end{align*}
where $\sum_{\cZ^k_{pred}}$ means summation over all possible value of $\cZ^k_{pred}$.

Therefore, by Markov Inequality, for any fixed $M\in\cM$ and fixed $k\in[K]$, and arbitrary $\delta\in(0,1)$, we have:
\begin{align*}
    \Pr(\sum_{i=1}^k \log \frac{f^{\pi^i,\tpi^i}_M(\cZ^i)}{f^{\pi^i,\tpi^i}_{M^*}(\cZ^i)} \geq \log \frac{1}{\delta}) \leq \delta \cdot \EE[\exp(\sum_{i=1}^k \log \frac{f^{\pi^i,\tpi^i}_M(\cZ^i)}{f^{\pi^i,\tpi^i}_{M^*}(\cZ^i)})] = \delta.
\end{align*}
By taking union bound over all $M \in \cM$ and all $k\in[K]$, we have:
\begin{align*}
    \Pr(\max_{M\in\cM} \sum_{i=1}^k\log \frac{f^{\pi^i,\tpi^i}_M(\cZ^i)}{f^{\pi^i,\tpi^i}_{M^*}(\cZ^i)} \geq \log \frac{|\cM|K}{\delta}) \leq \delta,\quad \forall k \in [K].
\end{align*}
\end{proof}
\noindent Given dataset $D^k:=\{(\pi^i, \tpi^i, \cZ^i)\}_{i=1}^k$, we use $\bD^k$ to denote the ``tangent'' sequence $\{(\pi^i,\tpi^i,\bar{\cZ}^i)\}_{i=1}^k$ where the policies are the same as $D^k$ while each $\bar{\cZ}^i$ is independently sampled from the same distribution as $\cZ^i$ conditioning on $\pi^i$ and $\tpi^i$.
\begin{lemma}\label{lem:expectation_of_loss}
    Let $l:\Pi\times\Pi\times(\cS\times\cA\times\cS)^H\times (\cS\times\cA\times\cS)^H \rightarrow \mR$ be a real-valued loss function which maps from the joint space of two policies and space of $\cZ^k$ to $\mR$. Define $L(D^k):=\sum_{i=1}^k l(\pi^i,\tpi^i,\cZ^i)$ and $L(\bD^k):=\sum_{i=1}^k l(\pi^i,\tpi^i,\bar{\cZ}^i)$. Then, for arbitrary $k\in[K]$,
    \begin{align*}
        \EE[\exp(L(D^k) - \log \EE[\exp(L(\bD^k))|D^k])] = 1.
    \end{align*}
\end{lemma}
\begin{proof}
    We denote $E^i := \EE_{\cZ^i}[\exp(l(\pi^i,\tpi^i,\cZ^i))|\pi^i,\tpi^i,M^*]$. By definition of $\bar{\cZ}^i$, we should also have:
    \begin{align*}
        \EE_{\bD^k}[\exp(\sum_{i=1}^k l(\pi^i,\tpi^i,\bar{\cZ}^i))|D^k] = \prod_{i=1}^k E^i.
    \end{align*}
    Therefore,
    \begin{align*}
        &\EE_{D^k}[\exp(L(D^k) - \log \EE_{\bD^k}[\exp(L(\bD^k))|D^k])]\\
        =&\EE_{D^{k-1}\cup\{\pi^k,\tpi^k\}}[\EE_{\cZ^k}[\frac{\exp(\sum_{i=1}^{k}l(\pi^i,\tpi^i,\cZ^i))}{\EE_{\bD^k}[\exp(\sum_{i=1}^k l(\pi^i,\tpi^i,\bar{\cZ}^i))|D^k]} |D^{k-1}\cup\{\pi^k,\tpi^k\}]]\\
        =&\EE_{D^{k-1}\cup\{\pi^k,\tpi^k\}}[\EE_{\cZ^k}[\frac{\exp(\sum_{i=1}^{k}l(\pi^i,\tpi^i,\cZ^i))}{\prod_{i=1}^k E^i} |D^{k-1}\cup\{\pi^k,\tpi^k\}]]\\
        =&\EE_{D^{k-1}\cup\{\pi^k,\tpi^k\}}[\frac{\exp(\sum_{i=1}^{k-1}l(\pi^i,\tpi^i,\cZ^i))}{\prod_{i=1}^{k-1} E^i} \cdot \EE_{\cZ^k}[\frac{l(\pi^k,\tpi^k,\cZ^k)}{E^k} |D^{k-1}\cup\{\pi^k,\tpi^k\}]]\\
        =&\EE_{D^{k-1}\cup\{\pi^k,\tpi^k\}}[\frac{\exp(\sum_{i=1}^{k-1}l(\pi^i,\tpi^i,\cZ^i))}{\prod_{i=1}^{k-1} E^i}] \\
        =&\EE_{D^{k-1}}[\frac{\exp(\sum_{i=1}^{k-1}l(\pi^i,\tpi^i,\cZ^i))}{\prod_{i=1}^{k-1} E^i}]  = ... = 1.
    \end{align*}
\end{proof}
\ThmMLE*
\begin{proof}
Given a model $M\in\cM$, we consider the loss function:
\begin{align*}
    l_M(\pi, \tpi, \cZ) := 
        \begin{cases}
            \frac{1}{2}\log\frac{f^{\pi,\tpi}_M(\cZ)}{f^{\pi,\tpi}_{M^*}(\cZ)},& \text{if}~ f^{\pi,\tpi}_{M^*}(\cZ) \neq 0\\
            0,              & \text{otherwise}
        \end{cases}
\end{align*}
Define $M_{\MLE}^k \gets \arg\max_{M\in\cM} l_{\MLE}^k(M)$. Considering the event $\cE$:
\begin{align*}
    \cE := \{l_{\MLE}^k(M_{\MLE}^k) - l_{\MLE}^k(M^*) \leq \log \frac{2|\cM|KH}{\delta},\quad \forall k\in[K] \}.
\end{align*}
and the event $\cE'$ defined by:
\begin{align*}
    \cE' := \{-\log \EE_{\bD^k}[\exp L_M(\bD^k)|D^k] \leq - L_M(D^k) + \log (\frac{2|\cM|KH}{\delta}),\quad \forall M\in\cM, k\in[K]\}.
\end{align*}
where we define $L_M(D^k):=\sum_{i=1}^k l_M(\pi^i,\tpi^i,\cZ^i)$ and $L_M(\bD^k):=\sum_{i=1}^k l_M(\pi^i,\tpi^i,\bar{\cZ}^i)$. 
By Lem.~\ref{lem:MLE_Diff}, we have $\Pr(\cE) \geq 1-\frac{\delta}{2H}$.
Besides, by applying Lem.~\ref{lem:expectation_of_loss} on $l_M$ defined above and applying Markov inequality and the union bound over all $M\in\cM$ and $k\in[K]$, we have $\Pr(\cE') \geq 1-\frac{\delta}{2H}$.

On the event $\cE \cap \cE'$, for any $k\in[K]$, we have $M^* \in \hat\cM^k$, and for any $M \in \hat\cM^k$:
\begin{align*}
    -\log \EE_{\bD^k}[\exp L_{M}(\bD^k)|D^k] \leq& - L_{M}(D^k) + \log (\frac{2|\cM|KH}{\delta}) \\
    = & l_{\MLE}^k(M^*) - l_{\MLE}^k(M) + \log (\frac{2|\cM|KH}{\delta}) \\
    \leq & l_{\MLE}^k(M^k_\MLE) - l_{\MLE}^k(M) + \log (\frac{2|\cM|KH}{\delta}) \\
    \leq & 2 \log (\frac{2|\cM|KH}{\delta}).
\end{align*}
Therefore, for any $k$ and any $M\in\hat\cM^k$,
\begin{align*}
    2 \log (\frac{2|\cM|KH}{\delta}) \geq & -\sum_{i=1}^k \log \EE_{\cZ^i}[\sqrt{\frac{f^{\pi^i,\tpi^i}_{M}(\cZ^i)}{f^{\pi^i,\tpi^i}_{M^*}(\cZ^i)}}|\pi^i,\tpi^i,M^*]\\
    \geq & \sum_{i=1}^k 1 - \EE_{\cZ^i}[\sqrt{\frac{f^{\pi^i,\tpi^i}_{M}(\cZ^i)}{f^{\pi^i,\tpi^i}_{M^*}(\cZ^i)}}|\pi^i,\tpi^i,M^*] \tag{$-\log x \geq 1-x$}\\
    =&\sum_{i=1}^k \EE_{\cZ^i_{cond}}[1 - \sum_{\cZ^i_{pred}} \sqrt{f^{\pi^i,\tpi^i}_{M}(\cZ^i)f^{\pi^i,\tpi^i}_{M^*}(\cZ^i)}|\pi^i,\tpi^i,M^*].
\end{align*}
For any $i\in[k]$ and for arbitrary random variable $s_h^i,a_h^i \in \cZ^i_{cond}$ and $s_{h+1}'^i \in \cZ^i_{pred}$, we have:
\begin{align*}
    &\EE_{\cZ^i_{cond}}[1 - \sum_{\cZ^i_{pred}} \sqrt{f^{\pi^i,\tpi^i}_{M}(\cZ^i)f^{\pi^i,\tpi^i}_{M^*}(\cZ^i)}|\pi^i,\tpi^i,M^*] \\
    =& \EE_{\cZ^i_{cond}}[1 - \SumInt_{s_{h+1}'^i}\sqrt{\mP_{T,h}(s_{h+1}'^i|s_h^i,a_h^i,\mu^{\pi^i}_{M,h})\mP_{T^*,h}(s_{h+1}'^i|s_h^i,a_h^i,\mu^{\pi^i}_{M^*,h})} \sum_{\cZ^i_{pred}\setminus\{s_{h+1}'^i\}} \sqrt{f^{\pi^i,\tpi^i}_{M}(\cZ^i)f^{\pi^i,\tpi^i}_{M^*}(\cZ^i)}|\pi^i,\tpi^i,M^*] \tag{Independence between $s_{h+1}'^i$ and $\cZ^i\setminus\{s_{h+1}'^i\}$ conditioning on $\cZ^i_{cond}$} \\
    \geq & \EE_{\cZ^i_{cond}}[1 - \SumInt_{s_{h+1}'^i}\sqrt{\mP_{T,h}(s_{h+1}'^i|s_h^i,a_h^i,\mu^{\pi^i}_{M,h})\mP_{T^*,h}(s_{h+1}'^i|s_h^i,a_h^i,\mu^{\pi^i}_{M^*,h})}|\pi^i,\tpi^i,M^*] \tag{$\sqrt{ab} \leq \frac{a+b}{2}$}\\
    = & \EE_{s_h^i,a_h^i}[1 - \SumInt_{s_{h+1}'^i}\sqrt{\mP_{T,h}(s_{h+1}'^i|s_h^i,a_h^i,\mu^{\pi^i}_{M,h})\mP_{T^*,h}(s_{h+1}'^i|s_h^i,a_h^i,\mu^{\pi^i}_{M^*,h})}|\pi^i,\tpi^i,M^*]\\
    = & \EE_{\pi^i,M^*}[\mH^2(\mP_{T,h}(\cdot|s_h^i,a_h^i,\mu^{\pi^i}_{M,h}),~\mP_{T^*,h}(\cdot|s_h^i,a_h^i,\mu^{\pi^i}_{M^*,h}))].
\end{align*}
Similarly, for arbitrary random variable $\ts_h^i,\ta_h^i \in \cZ^i_{cond}$ and $\ts_{h+1}'^i \in \cZ^i_{pred}$, we have:
\begin{align*}
    &\EE_{\cZ^i_{cond}}[1 - \sum_{\cZ^i_{pred}} \sqrt{f^{\pi^i,\tpi^i}_{M}(\cZ^i)f^{\pi^i,\tpi^i}_{M^*}(\cZ^i)}|\pi^i,\tpi^i,M^*]\geq \EE_{\tpi^i,M^*|\vecmu^{\pi^i}_{M^*}}[\mH^2(\mP_{T,h}(\cdot|\ts_h^i,\ta_h^i,\mu^{\pi^i}_{M,h}),~\mP_{T^*,h}(\cdot|\ts_h^i,\ta_h^i,\mu^{\pi^i}_{M^*,h}))].
\end{align*}
Therefore, on the event $\cE'$, for any $k\in[K]$, $M\in\hat\cM^{k}$, and a fixed $h\in[H]$, we have:
\begin{align*}
    2 \log (\frac{2|\cM|KH}{\delta})\geq & \sum_{i=1}^k \EE_{\pi^i,M^*}[\mH^2(\mP_{T,h}(\cdot|s_h^i,a_h^i,\mu^{\pi^i}_{M,h}),~\mP_{T^*,h}(\cdot|s_h^i,a_h^i,\mu^{\pi^i}_{M^*,h}))]\\
    2 \log (\frac{2|\cM|KH}{\delta}) \geq & \sum_{i=1}^k \EE_{\tpi^i,M^*|\vecmu^{\pi^i}_{M^*}}[\mH^2(\mP_{T,h}(\cdot|\ts_h^i,\ta_h^i,\mu^{\pi^i}_{M,h}),~\mP_{T^*,h}(\cdot|\ts_h^i,\ta_h^i,\mu^{\pi^i}_{M^*,h}))].
\end{align*}
By taking the union bound for all $h\in[H]$, we finish the proof for DCP of MFG. The analysis and results for MFC is similar and easier so we omit it here.
\end{proof}
\begin{corollary}\label{coro:TV_square_bound}
    Under the same event in Thm.~\ref{thm:MLE_Gaurantee}, for any $k\in[K]$, $M\in\hat\cM^{k}$, and a fixed $h\in[H]$, we have:
    \begin{align*}
        \sum_{i=1}^k \EE_{\pi^i,M^*}[\TV^2(\mP_{T,h}(\cdot|s_h^i,a_h^i,\mu^{\pi^i}_{M^*,h}),~\mP_{T^*,h}(\cdot|s_h^i,a_h^i,\mu^{\pi^i}_{M^*,h}))] \leq& (4 + 8 L_T^2 H^2) \log (\frac{2|\cM|KH}{\delta}),\\
        \sum_{i=1}^k \EE_{\tpi^i,M^*|\vecmu^{\pi^i}_{M^*}}[\TV^2(\mP_{T,h}(\cdot|\ts_h^i,\ta_h^i,\mu^{\pi^i}_{M^*,h}),~\mP_{T^*,h}(\cdot|\ts_h^i,\ta_h^i,\mu^{\pi^i}_{M^*,h}))] \leq& (4 + 8 L_T^2 H^2) \log (\frac{2|\cM|KH}{\delta}).
    \end{align*}
\end{corollary}
\begin{proof}
    By Assump.~\ref{assump:lipschitz}, for any $i$, we have:
    \begin{align*}
        & \EE_{\pi^i,M^*}[\TV^2(\mP_{T,h}(\cdot|s_h^i,a_h^i,\mu^{\pi^i}_{M^*,h}),~\mP_{T^*,h}(\cdot|s_h^i,a_h^i,\mu^{\pi^i}_{M^*,h}))] \\
        \leq & 2\EE_{\pi^i,M^*}[\TV^2(\mP_{T,h}(\cdot|s_h^i,a_h^i,\mu^{\pi^i}_{M,h}),~\mP_{T^*,h}(\cdot|s_h^i,a_h^i,\mu^{\pi^i}_{M^*,h}))] + 2L_T^2\|\mu^{\pi^i}_{M,h} - \mu^{\pi^i}_{M^*,h}\|_\TV^2 \\
        \leq & 2\EE_{\pi^i,M^*}[\TV^2(\mP_{T,h}(\cdot|s_h^i,a_h^i,\mu^{\pi^i}_{M,h}),~\mP_{T^*,h}(\cdot|s_h^i,a_h^i,\mu^{\pi^i}_{M^*,h}))] \\
        & + 4L_T^2 H \EE_{\pi,M}[\sum_{\ph=1}^{h-1} \TV^2(\mP_{T,\ph}(\cdot|s_\ph^i,a_\ph^i,\mu^{\pi^i}_{M,\ph}),~\mP_{T^*,\ph}(\cdot|s_\ph^i,a_\ph^i,\mu^{\pi^i}_{M^*,\ph}))]. \tag{Lem.~\ref{lem:density_est_err}; Cauchy-Schwarz inequality;} 
    \end{align*}
    Therefore, on the event $\cE'$, for any $k\in[K]$, $M\in\hat\cM^k$, and a fixed $h\in[H]$, we have:
    \begin{align*}
        \sum_{i=1}^k \EE_{\pi^i,M^*}[\TV^2(\mP_{T,h}(\cdot|s_h^i,a_h^i,\mu^{\pi^i}_{M^*,h}),~\mP_{T^*,h}(\cdot|s_h^i,a_h^i,\mu^{\pi^i}_{M^*,h}))] \leq (4 + 8 L_T^2 H^2) \log (\frac{2|\cM|KH}{\delta}).
    \end{align*}
    Similarly, we have:
    \begin{align*}
        &\EE_{\tpi^i,M^*|\vecmu^{\pi^i}_{M^*}}[\TV^2(\mP_{T,h}(\cdot|\ts_h^i,\ta_h^i,\mu^{\pi^i}_{M,h}),~\mP_{T^*,h}(\cdot|\ts_h^i,\ta_h^i,\mu^{\pi^i}_{M^*,h}))]\\
        \leq & 2\EE_{\tpi^i,M^*|\vecmu^{\pi^i}_{M^*}}[\TV^2(\mP_{T,h}(\cdot|\ts_h^i,\ta_h^i,\mu^{\pi^i}_{M^*,h}),~\mP_{T^*,h}(\cdot|\ts_h^i,\ta_h^i,\mu^{\pi^i}_{M^*,h}))] + 2 L_T^2 \|\mu^{\pi^i}_{M,h} - \mu^{\pi^i}_{M^*,h}\|_\TV^2.
    \end{align*}
    By similar discussion, we have:
    \begin{align*}
        \sum_{i=1}^k \EE_{\tpi^i,M^*|\vecmu^{\pi^i}_{M^*}}[\TV^2(\mP_{T,h}(\cdot|\ts_h^i,\ta_h^i,\mu^{\pi^i}_{M^*,h}),~\mP_{T^*,h}(\cdot|\ts_h^i,\ta_h^i,\mu^{\pi^i}_{M^*,h}))] \leq (4 + 8 L_T^2 H^2) \log (\frac{2|\cM|KH}{\delta}).
    \end{align*}
\end{proof}

\begin{restatable}{theorem}{ThmAccuDiff}[Accumulative Model Difference]\label{thm:accumu_model_diff}
    For any $\delta \in (0,1)$, with probability $1-3\delta$, for any sequence $\{\hM^{k+1}\}_{k\in[K]} $ with $\hat M^{k+1}\in \hat\cM^{k+1}$ for all $k\in[K]$, and any $h\in[H]$, we have:
    \begin{align*}
        &\sum_{k=1}^{K} \EE_{\pi^{k+1},M^*}[\|\mP_{\hat{T}^{k+1},h}(\cdot|s_h,a_h,\mu^{\pi^{k+1}}_{M^*,h})- \mP_{T^*,h}(\cdot|s_h,a_h,\mu^{\pi^{k+1}}_{M^*,h})\|_\TV]\\
        & \quad\quad = O\Big((1 + L_TH)\sqrt{ K \dimE_\alpha(\cM, \epsilon_0) \log\frac{2|\cM|KH}{\delta} }+ \alpha K\epsilon_0\Big)\\
        &\sum_{k=1}^{K} \EE_{\tpi^{k+1},M^*|\vecmu^{\pi^{k+1}}_{M^*}}[\|\mP_{\hat{T}^{k+1},h}(\cdot|s_h,a_h,\mu^{\pi^{k+1}}_{M^*,h}) - \mP_{T^*,h}(\cdot|s_h,a_h,\mu^{\pi^{k+1}}_{M^*,h}) \|_\TV]\\
        & \quad\quad =  O\Big((1 + L_TH)\sqrt{ K \dimE_\alpha(\cM, \epsilon_0) \log\frac{2|\cM|KH}{\delta} } + \alpha K\epsilon_0\Big).
    \end{align*}
\end{restatable}
\begin{proof}
    We first take a look at the data $(\ts_h^k,\ta_h^k,\ts_{h+1}'^k)$ collected by $(\tpi^i,\pi^i)$ and the Eluder Dimension w.r.t. the Hellinger distance.
    On the event in Thm.~\ref{thm:MLE_Gaurantee} (which implies Corollary~\ref{coro:TV_square_bound}) and Lem.~\ref{lem:concentration}, there exists an absolute constant $c_\TV$, s.t., w.p. $1-\frac{\delta}{2}$, for any $h\in[H]$, and any $\hM^{k+1} \in \hat\cM^{k+1}$, we have:
    \begin{align*}
        \sum_{i=1}^{k} \TV^2(\mP_{T^*,h}(\cdot|\ts_h^i,\ta_h^i,\mu^{\pi^i}_{M^*,h}), \mP_{\hat{T}^{k+1},h}(\cdot|\ts_h^i,\ta_h^i,\mu^{\pi^i}_{M^*,h})) \leq c_\TV(1 + L_T^2H^2)\log\frac{2|\cM|KH}{\delta}.\numberthis\label{eq:TV_Hellinger}
    \end{align*}
    By Lem.~\ref{lem:Eluder_Bound}, there exists some constant $c_\TV'$, for any $\epsilon_0$, we have:
    \begin{align*}
        \sum_{k=1}^{K} \TV(\mP_{T^*,h}(\cdot|\ts_h^{k+1},\ta_h^{k+1},\mu^{\pi^{k+1}}_{M^*,h}),& \mP_{\hat{T}^{k+1},h}(\cdot|\ts_h^{k+1},\ta_h^{k+1},\mu^{\pi^{k+1}}_{M^*,h})) \\
        & \leq c_\TV'\Big((1 + L_TH)\sqrt{K \dimE_\alpha(\cM, \epsilon_0) \log\frac{2|\cM|KH}{\delta}} + \alpha K \epsilon_0\Big).
    \end{align*}
    By applying Lem.~\ref{lem:concentration} again, w.p. $1-\frac{\delta}{2}$, we have:
    \begin{align*}
        &\sum_{k=1}^{K} \EE_{\tpi^{k+1},M^*|\vecmu^{\pi^{k+1}}_{M^*}}[\TV(\mP_{T^*,h}(\cdot|s_h,a_h,\mu^{\pi^{k+1}}_{M^*,h}) , \mP_{\hat{T}^{k+1},h}(\cdot|s_h,a_h,\mu^{\pi^{k+1}}_{M^*,h}))]\\
        \leq & 3c_\TV'\Big((1 + L_TH)\sqrt{ K \dimE_\alpha(\cM, \epsilon_0) \log\frac{2|\cM|KH}{\delta}} + \alpha K \epsilon_0\Big) + \log\frac{2|\cM|H}{\delta} \\
        \leq & (3c_\TV'+1)\Big((1 + L_TH)\sqrt{ K \dimE_\alpha(\cM, \epsilon_0) \log\frac{2|\cM|KH}{\delta}} + \alpha K \epsilon_0\Big).
    \end{align*}
    Combine them together, for some constant $c$, we have:
    \begin{align*}
        &\sum_{k=1}^{K} \EE_{\tpi^{k+1},M^*|\vecmu^{\pi^{k+1}}_{M^*}}[ \|\mP_{T^*,h}(\cdot|s_h,a_h,\mu^{\pi^{k+1}}_{M^*,h}), \mP_{\hat{T}^{k+1},h}(\cdot|s_h,a_h,\mu^{\pi^{k+1}}_{M^*,h})\|_\TV]\\
        \leq & (3c+1)\Big((1 + L_TH)\sqrt{ K \dimE_\alpha(\cM,\epsilon_0) \log\frac{2|\cM|KH}{\delta}} + \alpha K\epsilon_0 \Big).
    \end{align*}
    where we use that $\dim_\alpha(\cM,\epsilon_0) = \min\{\dimE_\alpha(\cM,\mH,\epsilon_0), \dimE_\alpha(\cM,\matDTV,\epsilon_0)\}$.
    
    Then, we can conduct similar discussion for the data $(s_h^i,a_h^i,s_{h+1}^i)$ collected by $(\pi^{k+1},\pi^{k+1})$, and for some constant $c'$, we have:
    \begin{align*}
        &\sum_{k=1}^{K} \EE_{\pi^{k+1},M^*}[ \|\mP_{T^*,h}(\cdot|s_h,a_h,\mu^{\pi^{k+1}}_{M^*,h}), \mP_{\hat{T}^{k+1},h}(\cdot|s_h,a_h,\mu^{\pi^{k+1}}_{M^*,h})\|_\TV]\\
        \leq & (3c'+1)\Big((1 + L_TH)\sqrt{ K \dimE_\alpha(\cM,\epsilon_0) \log\frac{2|\cM|KH}{\delta}} + \alpha K\epsilon_0 \Big).
    \end{align*}
    We finish the proof by noting that the total failure rate can be upper bounded by $\delta + \delta / 2 \cdot 2 \cdot 2 = 3\delta$.
\end{proof}
\section{Proofs for Mean-Field Reinforcement Learning}\label{appx:proofs_for_MFRL}

\subsection{Missing Details}\label{appx:missing_details}
\begin{algorithm}
    \textbf{Input}: Policy sequence $\pi^1,...,\pi^K$; Accuracy level $\epsilon$; Confidence level $\delta$.\\
    $N \gets \lceil \log_\frac{3}{2}\frac{1}{\delta}\rceil$.\\
    Randomly select $N$ policies from $\pi^1,...,\pi^K$, denoted as $\pi^{k_1},...\pi^{k_N}$.\\
    \For{$n\in[N]$}{
        Sample $\frac{16}{\epsilon^2}\log\frac{2N}{\delta}$ trajectories by deploying $\pi^{k_n}$.\\
        Compute empirical estimation $\hat{J}_{M^*}(\pi^{k_n})$ by averaging the return in trajectories.
    }
    \Return $\pi := \pi^{k_{n^*}}$ with $n^* \gets \arg\max_{n\in[N]}\hat{J}_{M^*}(\pi^{k_n})$.
    \caption{Regret to PAC Conversion}\label{alg:Regret2PAC}
\end{algorithm}

\subsection{Proofs for Basic Lemma}
\begin{restatable}{lemma}{ThmDensityErr}[Density Estimation Error]\label{lem:density_est_err}
    Given two model $M$ and $\tM$ and a policy $\pi$, we have:
    \begin{align}
        \|\mu^\pi_{M,h+1} - \mu^\pi_{\tM,h+1}\|_\TV \leq& \EE_{\pi,M}[\sum_{\ph=1}^h \|\mP_{T,\ph}(\cdot|s_\ph,a_\ph,\mu^\pi_{M,\ph}) - \mP_{\tT,\ph}(\cdot|s_\ph,a_\ph,\mu^\pi_{\tM,\ph})\|_\TV].\label{eq:density_diff1}
    \end{align}
    Besides, under Assump.~\ref{assump:lipschitz}, we have:
    \begin{align}
        \|\mu^\pi_{M,h+1} - \mu^\pi_{\tM,h+1}\|_\TV \leq \EE_{\pi,M}[\sum_{\ph=1}^h (1+L_T)^{h-\ph} \|\mP_{T,\ph}(\cdot|s_\ph,a_\ph,\mu^\pi_{M,\ph}) - \mP_{\tT,\ph}(\cdot|s_\ph,a_\ph,\mu^\pi_{M,\ph})\|_\TV].\label{eq:density_diff2}
    \end{align}
\end{restatable}
\begin{proof}
In the following, we will use $\bar{\cS}$ or $\bar{\cS}'$ to denote a subset of $\cS$.

\noindent \textbf{Proof for Eq.~\eqref{eq:density_diff1}}
\begin{align*}
    &\|\mu^\pi_{M,h+1} - \mu^\pi_{\tM,h+1}\|_{\TV} \\   
    =& \sup_{\bar{\cS} \subset \cS}|\SumInt_{s_{h+1} \in \bar{\cS}}\Big(\SumInt_{s_h,a_h}\mu_{M,h}^\pi(s_h)\pi(a_h|s_h)\mP_{T,h}(s_{h+1}|s_h,a_h,\mu^\pi_{M,h}) - \SumInt_{s_h,a_h}\mu^\pi_{\tM,h}(s_h)\pi(a_h|s_h)\mP_{\tT,h}(s_{h+1}|s_h,a_h,\mu^\pi_{\tM,h})\Big)|\\
    =& \sup_{\bar{\cS} \subset \cS}|\SumInt_{s_{h+1} \in \bar{\cS}}\SumInt_{s_h,a_h}(\mu_{M,h}^\pi(s_h)-\mu^\pi_{\tM,h}(s_h))\pi(a_h|s_h)\mP_{\tT,h}(s_{h+1}|s_h,a_h,\mu^\pi_{\tM,h})|\\
    & +\sup_{\bar{\cS}' \subset \cS}|\SumInt_{s_{h+1} \in \bar{\cS}'}\SumInt_{s_h,a_h}\mu^\pi_{M,h}(s_h)\pi(a_h|s_h)(\mP_{T,h}(s_{h+1}|s_h,a_h,\mu^\pi_{M,h}) -\mP_{\tT,h}(s_{h+1}|s_h,a_h,\mu^\pi_{\tM,h}))|.
\end{align*}
For the first part, we have:
\begin{align*}
    &\sup_{\bar{\cS} \subset \cS}|\SumInt_{s_{h+1} \in \bar{\cS}}\SumInt_{s_h,a_h}(\mu_{M,h}^\pi(s_h)-\mu^\pi_{\tM,h}(s_h))\pi(a_h|s_h)\mP_{\tT,h}(s_{h+1}|s_h,a_h,\mu^\pi_{M,h})| \\
    \leq & \sup_{\bar{\cS} \subset \cS}|\SumInt_{s_h}(\mu_{M,h}^\pi(s_h)-\mu^\pi_{\tM,h}(s_h)) \SumInt_{a_h}\pi(a_h|s_h)\SumInt_{s_{h+1} \in \bar{\cS}}\mP_{\tT,h}(s_{h+1}|s_h,a_h,\mu^\pi_{M,h})| \\
    \leq & \sup_{\bar{\cS} \subset \cS}|\SumInt_{s_h\in \bar{\cS}}\mu_{M,h}^\pi(s_h)-\mu^\pi_{\tM,h}(s_h)|\\
    =&\|\mu_{M,h}^\pi-\mu^\pi_{\tM,h}\|_{\TV}.
\end{align*}
For the second part, we have:
\begin{align*}
    &\sup_{\bar{\cS}' \subset \cS}|\SumInt_{s_{h+1} \in \bar{\cS}'}\SumInt_{s_h,a_h}\mu^\pi_{M,h}(s_h)\pi(a_h|s_h)(\mP_{T,h}(s_{h+1}|s_h,a_h,\mu^\pi_{M,h}) -\mP_{\tT,h}(s_{h+1}|s_h,a_h,\mu^\pi_{\tM,h}))|\\
    \leq & \SumInt_{s_h,a_h}\mu^\pi_{M,h}(s_h)\pi(a_h|s_h)\sup_{\bar{\cS}' \subset \cS}|\SumInt_{s_{h+1} \in \bar{\cS}'}(\mP_{T,h}(s_{h+1}|s_h,a_h,\mu^\pi_{M,h}) -\mP_{\tT,h}(s_{h+1}|s_h,a_h,\mu^\pi_{\tM,h}))|\\
    =& \EE_{s_h\sim \mu_{M,h}^\pi, a_h\sim\pi(\cdot|s_h)}[\|\mP_{T,h}(\cdot|s_h,a_h,\mu^\pi_{M,h}) - \mP_{\tT,h}(\cdot|s_h,a_h,\mu^\pi_{\tM,h})\|_{\TV}].
\end{align*}
Therefore,
\begin{align*}
    \|\mu^\pi_{M,h+1} - \mu^\pi_{\tM,h+1}\|_{\TV} \leq&  \|\mu^\pi_{M,h} - \mu^\pi_{\tM,h}\|_{\TV} + \EE_{s_h\sim \mu_{M,h}^\pi, a_h\sim\pi(\cdot|s_h)}[\|\mP_{T,h}(\cdot|s_h,a_h,\mu^\pi_{M,h}) - \mP_{\tT,h}(\cdot|s_h,a_h,\mu^\pi_{\tM,h})\|_{\TV}] \\
    \leq & ... \leq \EE_{\pi,M}[\sum_{\ph=1}^h \|\mP_{T,\ph}(\cdot|s_\ph,a_\ph,\mu^\pi_{M,\ph}) - \mP_{\tT,\ph}(\cdot|s_\ph,a_\ph,\mu^\pi_{\tM,\ph})\|_{\TV}].\numberthis\label{eq:refer_eq}
\end{align*}
\textbf{Proof for Eq.~\eqref{eq:density_diff2}} \quad
Starting with the first inequality of Eq.~\eqref{eq:refer_eq} and applying the Assump.~\ref{assump:lipschitz}, we directly have:
\begin{align*}
    \|\mu^\pi_{M,h+1} - \mu^\pi_{\tM,h+1}\|_{\TV} \leq & (1+L_T)\|\mu^\pi_{M,h} - \mu^\pi_{\tM,h}\|_{\TV} + \EE_{s_h\sim \mu^\pi_{M,h}, a_h\sim \pi}[\|\mP_{T,h}(\cdot|s_h,a_h,\mu^\pi_{M,h}) - \mP_{\tT,h}(\cdot|s_h,a_h,\mu^\pi_{M,h})\|_{\TV}]\\
    \leq & \EE_{\pi}[\sum_{\ph=1}^h (1+L_T)^{h-\ph} \|\mP_{T,\ph}(\cdot|s_\ph,a_\ph,\mu^\pi_{M,\ph}) - \mP_{\tT,\ph}(\cdot|s_\ph,a_\ph,\mu^\pi_{M,\ph})\|_{\TV}].
\end{align*}
\end{proof}

\begin{lemma}\label{lem:density_est_err_2}
    Under Assump.~\ref{assump:lipschitz} and Assump.~\ref{assump:contraction}, we have:
    \begin{align}
        \|\mu^\pi_{M,h+1} - \mu^\pi_{\tM,h+1}\|_\TV \leq \EE_{\pi,M}[\sum_{\ph=1}^h L_\Gamma^{h-h'} \|\mP_{T,\ph}(\cdot|s_\ph,a_\ph,\mu^\pi_{M,\ph}) - \mP_{\tT,\ph}(\cdot|s_\ph,a_\ph,\mu^\pi_{M,\ph})\|_\TV].\label{eq:density_diff3}
    \end{align}
\end{lemma}
\begin{proof}
    Under Assump.~\ref{assump:contraction}, we can use a different way to decompose the density difference.
    \begin{align*}
        &\|\mu^\pi_{M,h+1} - \mu^\pi_{\tM,h+1}\|_{\TV} \\   
        =& \sup_{\bar{\cS} \subset \cS}|\SumInt_{s_{h+1} \in \bar{\cS}}\Big(\SumInt_{s_h,a_h}\mu_{M,h}^\pi(s_h)\pi(a_h|s_h)\mP_{T,h}(s_{h+1}|s_h,a_h,\mu^\pi_{M,h}) - \SumInt_{s_h,a_h}\mu^\pi_{\tM,h}(s_h)\pi(a_h|s_h)\mP_{\tT,h}(s_{h+1}|s_h,a_h,\mu^\pi_{\tM,h})\Big)|\\
        =& \sup_{\bar{\cS} \subset \cS}|\SumInt_{s_{h+1} \in \bar{\cS}}\Big(\SumInt_{s_h,a_h}\mu_{M,h}^\pi(s_h)\pi(a_h|s_h)\mP_{\tT,h}(s_{h+1}|s_h,a_h,\mu^\pi_{M,h}) - \SumInt_{s_h,a_h}\mu^\pi_{\tM,h}(s_h)\pi(a_h|s_h)\mP_{\tT,h}(s_{h+1}|s_h,a_h,\mu^\pi_{\tM,h})\Big)|\\
        &+ \sup_{\bar{\cS} \subset \cS}|\SumInt_{s_{h+1} \in \bar{\cS}}\SumInt_{s_h,a_h}\mu_{M,h}^\pi(s_h)\pi(a_h|s_h)\Big(\mP_{T,h}(s_{h+1}|s_h,a_h,\mu^\pi_{M,h}) - \mP_{\tT,h}(s_{h+1}|s_h,a_h,\mu^\pi_{M,h})\Big)|\\
        \leq& \|\Gamma^\pi_{\tM,h}(\mu_{M,h}^\pi) - \Gamma^\pi_{\tM,h}(\mu_{\tM,h}^\pi)\|_{\TV} + \EE_{s_h\sim \mu^\pi_{M,h}, a_h\sim \pi}[\|\mP_{T,h}(\cdot|s_h,a_h,\mu^\pi_{M,h}) - \mP_{\tT,h}(\cdot|s_h,a_h,\mu^\pi_{M,h})\|_{\TV}]\\
        \leq& L_\Gamma \|\mu_{M,h}^\pi - \mu^\pi_{\tM,h}\|_{\TV} + \EE_{s_h\sim \mu^\pi_h, a_h\sim \pi}[\|\mP_{T,h}(\cdot|s_h,a_h,\mu^\pi_{M,h}) - \mP_{\tT,h}(\cdot|s_h,a_h,\mu^\pi_{M,h})\|_{\TV}]\\
        \leq & \EE_{\pi}[\sum_{\ph=1}^h L_\Gamma^{h-h'} \|\mP_{T,\ph}(\cdot|s_\ph,a_\ph,\mu^\pi_{M,\ph}) - \mP_{\tT,\ph}(\cdot|s_\ph,a_\ph,\mu^\pi_{M,\ph})\|_{\TV}].
    \end{align*}
\end{proof}

\noindent As implied by Lem.~\ref{lem:density_est_err} and Lem.~\ref{lem:density_est_err_2}, we have the following corollary.
\begin{corollary}\label{coro:accum_density_err}
    In general,
    \begin{align*}
        \sum_{h=1}^{H} \|\mu^\pi_{M,h} - \mu^\pi_{\tM,h}\|_{\TV} 
        \leq &  \EE_{\pi,M}[\sum_{h=1}^{H} (H-h)\|\mP_{T,h}(\cdot|s_h,a_h,\mu^\pi_{M,h}) - \mP_{\tT,h}(\cdot|s_h,a_h,\mu^\pi_{\tM,h})\|_{\TV}].
    \end{align*}
    Besides, under Assump.~\ref{assump:lipschitz}, we have:
    \begin{align*}
        \sum_{h=1}^{H} \|\mu^\pi_{M,h} - \mu^\pi_{\tM,h}\|_\TV \leq & \sum_{h=1}^{H} \EE_{\pi,M}[\sum_{\ph=1}^{h-1} (1+L_T)^{h-\ph-1} \|\mP_{T,\ph}(\cdot|s_\ph,a_\ph,\mu^\pi_{M,\ph}) - \mP_{\tT,\ph}(\cdot|s_\ph,a_\ph,\mu^\pi_{M,\ph})\|_\TV]\\
        = & \sum_{h=1}^{H} \frac{(1 + L_T)^{H-h} - 1}{L_T}\EE_{\pi,M}[\|\mP_{T,h}(\cdot|s_h,a_h,\mu^\pi_{M,h}) - \mP_{\tT,h}(\cdot|s_h,a_h,\mu^\pi_{M,h})\|_\TV]
    \end{align*}
    Moreover, with additional Assump.~\ref{assump:contraction}, we have:
    \begin{align*}
        \sum_{h=1}^H \|\mu^\pi_{M,h} - \mu^\pi_{\tM,h}\|_\TV \leq & \sum_{h=1}^H \EE_{\pi,M}[\sum_{\ph=1}^{h-1} L_\Gamma^{h-h'-1} \|\mP_{T,\ph}(\cdot|s_\ph,a_\ph,\mu^\pi_{M,\ph}) - \mP_{\tT,\ph}(\cdot|s_\ph,a_\ph,\mu^\pi_{M,\ph})\|_\TV] \\
        \leq & \sum_{h=1}^{H} \frac{1 }{1 - L_\Gamma}\EE_{\pi,M}[\|\mP_{T,h}(\cdot|s_h,a_h,\mu^\pi_{M,h}) - \mP_{\tT,h}(\cdot|s_h,a_h,\mu^\pi_{M,h})\|_\TV]. \tag{$L_\Gamma < 1$}
    \end{align*}
\end{corollary}

\ThmModelDiff*
\begin{proof}
    By Assump.~\ref{assump:lipschitz}, we have:
    \begin{align*}
        &\Big|\EE_{\pi,M}[\sum_{h=1}^{H} \|\mP_{T,h}(\cdot|s_h,a_h,\mu^{\pi}_{M,h}) - \mP_{\tT,h}(\cdot|s_h,a_h,\mu^{\pi}_{M,h})\|_\TV] \\
        & \quad\quad - \EE_{\pi,M}[\sum_{h=1}^{H} \|\mP_{T,h}(\cdot|s_h,a_h,\mu^{\pi}_{M,h}) - \mP_{\tT,h}(\cdot|s_h,a_h,\mu^{\pi}_{\tM,h})\|_\TV] \Big| \leq  L_T \sum_{h=1}^{H} \|\mu^\pi_{M,h} - \mu^\pi_{\tM,h}\|_\TV.\numberthis\label{eq:sub_step_1}
    \end{align*}
    Then, by applying Corollary~\ref{coro:accum_density_err}, and plugging into the above equation, we can finish the proof.
\end{proof}

\ThmModelDiffPlus*
\begin{proof}
    By applying Eq.~\eqref{eq:sub_step_1} and Corollary~\ref{coro:accum_density_err}, we can finish the proof.
\end{proof}

\begin{lemma}[Concentration Lemma]\label{lem:concentration}
    Let $X_1,X_2,...$ be a sequence of random variable taking value in $[0,C]$ for some $C \geq 1$. Define $\cF_k = \sigma(X_1,..,X_{k-1})$ and $Y_k = \EE[X_k|\cF_k]$ for $k\geq 1$. For any $\delta > 0$, we have:
    \begin{align*}
        \Pr(\exists n \sum_{k=1}^n X_k \leq 3 \sum_{k=1}^n Y_k + C \log \frac{1}{\delta}) \leq \delta,\quad 
        \Pr(\exists n \sum_{k=1}^n Y_k \leq 3 \sum_{k=1}^n X_k + C \log \frac{1}{\delta}) \leq \delta.
    \end{align*}
\end{lemma}
\begin{proof}
    Define $Z_k := \EE[\exp(t\sum_{i=1}^k X_i - 3Y_i)]$. By taking $t\in[0,1/C]$, we have:
    \begin{align*}
        \EE[Z_k | \cF_k] =& \exp(t\sum_{i=1}^{k-1} (X_i - 3Y_i))\EE[\exp(t(X_k - 3 Y_k))|\cF_k] \\
        \leq & \exp(t\sum_{i=1}^{k-1} (X_i - 3Y_i)) \exp(-3Y_k)\EE[1 + tX_k + 2t^2 X_k^2 |\cF_k]\\
        \leq & \exp(t\sum_{i=1}^{k-1} (X_i - 3Y_i)) \exp(-3Y_k)\cdot (1 + 3tY_k) \tag{$0\geq tX_k \leq 1$}\\
        \leq & \exp(t\sum_{i=1}^{k-1} (X_i - 3Y_i)) \cdot \exp(-3Y_k + 3tY_k) \tag{$1 + x \leq \exp(x)$}\\
        \leq & \exp(t\sum_{i=1}^{k-1} (X_i - 3Y_i)) = Z_{k-1}.
    \end{align*}
    We augment the sequence by set $X_0 = Y_0 = 0$, which implies $Z_0 = 1$. Therefore, $\{Z_k\}_{k\geq 0}$ is a super-martingale w.r.t. $\{\cF_k\}_{k \geq 1}$. Denote $\tau$ to be the smallest $t$ such that $\sum_{i=1}^t (X_i - 3 Y_i) > C \log\frac{1}{\delta}$, we have:
    \begin{align*}
        Z_{k \wedge \tau} =& \EE[\exp(t\sum_{i=1}^{k \wedge \tau} (X_i - 3Y_i))] \\
        = & \EE[\sum_{j=1}^k \mathbb{I}[\tau = j] \exp(t\sum_{i=1}^{\tau} (X_i - 3Y_i))] + \EE[\mathbb{I}[\tau > k] \exp(t\sum_{i=1}^{k} (X_i - 3Y_i))] \\
        \leq & \exp(tC)\EE[\sum_{j=1}^k \mathbb{I}[\tau = j] \exp(t\sum_{i=1}^{\tau - 1} (X_i - 3Y_i))]+ \EE[\mathbb{I}[\tau > k] \exp(t\sum_{i=1}^{k} (X_i - 3Y_i))] \tag{$\exp(t(X_i - 3Y_i)) \leq \exp(tC)$} \\
        \leq & \exp(tC+tC\log\frac{1}{\delta})\sum_{j=1}^k \EE[\mathbb{I}[\tau = j]] + \exp(tC\log\frac{1}{\delta}) \EE[\mathbb{I}[\tau > k]] \\
        \leq & \exp(tC + tC \log\frac{1}{\delta}).
    \end{align*}
    which is upper bounded. Therefore, by the optimal stopping theorem, and choosing $t = 1/C$, we have:
    \begin{align*}
        \Pr(\exists k \leq K,~\sum_{i=1}^k X_k - 3 Y_k \geq C\log\frac{1}{\delta}) =& \Pr(\tau \leq K)  \leq \Pr(Z_{K\wedge \tau} \geq \exp(tl\log\frac{1}{\delta}))\\
        \leq & \frac{\EE[Z_{K\wedge \tau}]}{\exp(tC\log\frac{1}{\delta})}\leq  \frac{Z_0}{\exp(tC\log\frac{1}{\delta})} = \delta.
    \end{align*}
    Since the above holds for arbitrary $K$, by setting $K \rightarrow +\infty$, we have:
    $$
    \Pr(\exists n \sum_{k=1}^n X_k \leq 3 \sum_{k=1}^n Y_k + C \log \frac{1}{\delta}) \leq \delta.
    $$
    The other inequality can be proved similarly by considering $Z_k' = \EE[\exp(t\sum_{i=1}^k (Y_k - 3X_k)]$.
\end{proof}

\subsection{Proofs for RL for Mean-Field Control}\label{appx:proofs_for_MFC}

\LemValDecomp*
\begin{proof}
    We first prove the value difference for the general case. The lemma can be proved by directly assign $\tM = M^*$ and $\pi = \tpi$.
\begin{align*}
    &|J_M(\tpi;\vecmu^\pi_M) - J_\tM(\tpi;\vecmu^\pi_\tM)| \\
    =&|\EE_{s_1\sim\mu_1}[V^\tpi_{M,1}(s_1;\vecmu^\pi_M) - V^\tpi_{\tM,1}(s_1;\vecmu^\pi_\tM)]| \\
    = & |\EE_{s_1\sim\mu_1,a_1\sim \tpi}[r_{1}(s_1,a_1,\mu^\pi_{M,1}) - r_{1}(s_1,a_1,\mu^\pi_{\tM,1}) \\
    &\quad\quad + \SumInt_{s_2}\mP_{T,1}(s_2|s_1,a_1,\mu^\pi_{M,1})V^\tpi_{M,2}(s_2;\vecmu^\pi_M) - \SumInt_{s_2}\mP_{\tT,1}(s_2|s_1,a_1,\mu^\pi_{\tM,1})V^\tpi_{\tM,2}(s_2;\vecmu^\pi_\tM)]| \\
    \leq & L_r\|\mu^\pi_{M,1} - \mu^\pi_{\tM,1}\|_\TV + |\EE_{s_1\sim\mu_1,a_1\sim\tpi}[\SumInt_{s_2} \Big(\mP_{T,1}(s_2|s_1,a_1,\mu^\pi_{M,1}) - \mP_{\tT,1}(s_2|s_1,a_1,\mu^\pi_{\tM,1})\Big) V^\tpi_{\tM,2}(s_2;\vecmu^\pi_\tM)]| \\
    & + |\EE_{s_1\sim\mu_1,a_1\sim\tpi}[\SumInt_{s_2} \mP_{T,1}(s_2|s_1,a_1,\mu^\pi_{M,1}) \Big(V^\tpi_{M,2}(s_2;\vecmu^\pi_M) - V^\tpi_{\tM,2}(s_2;\vecmu^\pi_\tM)\Big)]| \\
    \leq & L_r\|\mu^\pi_{M,1} - \mu^\pi_{\tM,1}\|_\TV + \EE_{s_1\sim\mu_1,a_1\sim\tpi}[\|\mP_{T,1}(\cdot|s_1,a_1,\mu^\pi_{M,1}) - \mP_{\tT,1}(\cdot|s_1,a_1,\mu^\pi_{\tM,1})\|_\TV] \\
    &  + |\EE_{s_1\sim\mu_1,a_1\sim\tpi,s_2\sim\mP_{T,1}(\cdot|s_1,a_1,\mu^\pi_{M})}[V^\tpi_{M,2}(s_2;\vecmu^\pi_M) - V^\tpi_{\tM,2}(s_2;\vecmu^\pi_\tM)]| \\
    \leq & \sum_{h=1}^H L_r\|\mu^\pi_{M,h} - \mu^\pi_{\tM,h}\|_\TV + \EE_{\tpi,M|\vecmu^\pi_M}[\sum_{h=1}^H \|\mP_{T,h}(\cdot|s_h,a_h,\mu^\pi_{M,h}) - \mP_{\tT,h}(\cdot|s_h,a_h,\mu^\pi_{\tM,h})\|_\TV].\numberthis\label{eq:value_diff_2}
\end{align*}
we finish the proof by applying Corollary~\ref{coro:accum_density_err}.
\end{proof}

\begin{theorem}[Result for MFC; Full Version of Thm.~\ref{thm:short_main_results_MFC_MFG} and Thm.~\ref{thm:short_main_results_MFC_MFG_contractivity}]\label{thm:MFC_main_full}
    Under Assump.\ref{assump:realizability},~\ref{assump:lipschitz}, by running Alg.~\ref{alg:MLE_MB_Alg} with the MFC branch, after consuming $HK$ trajectories in Alg.~\ref{alg:MLE_MB_Alg} and additional $O(\frac{1}{\epsilon^2}\log^2\frac{1}{\delta})$ trajectories in the policy selection process in Alg.~\ref{alg:Regret2PAC}, where $K$ is set to
    \begin{align*}
        K = \tilde{O}\Big((1+L_rH)^2(1+L_TH)^2 \Big(\frac{(1+L_T)^H - 1}{L_T}\Big)^2 \frac{\dimE_\alpha(\cM,\epsilon_0)}{\epsilon^2}\Big)
    \end{align*}
    with
    \begin{align*}
        \epsilon_0 = O(\frac{L_T \epsilon}{\alpha H(1+L_rH)(1+L_TH)((1+L_T)^H - 1)}).
    \end{align*}
    or set to the following under additional Assump.~\ref{assump:contraction}:
    \begin{align*}
        K = \tilde{O}\Big((1+L_rH)^2(1+L_TH)^2 \Big(1 + \frac{L_T}{1 - L_\Gamma}\Big)^2 \frac{\dimE_\alpha(\cM,\epsilon_0)}{\epsilon^2}\Big),
    \end{align*}
    with
    \begin{align*}
        \epsilon_0 = O(\frac{\epsilon}{\alpha H(1+L_rH)(1+L_TH)} (1+\frac{L_T}{1-L_\Gamma})^{-1}).
    \end{align*}
    with probability at least $1-5\delta$, we have $\cE_\Opt(\hat\pi_\Opt^*) \leq \epsilon$.
\end{theorem}
\begin{proof}
    On the event of Thm.~\ref{thm:MLE_Gaurantee}, by Lem.~\ref{lem:value_decomposition}, we have:
    \begin{align*}
        &\sum_{k=1}^K \cE_\Opt(\pi^{k+1}) \leq \sum_{k=1}^K J_{M^{k+1}}(\pi^{k+1}) - J_{M^*}(\pi^{k+1}) \tag{$M^* \in \hat\cM^{k+1}$} \\
        \leq & \sum_{k=1}^K \EE_{\pi^{k+1},M^*}[\sum_{h=1}^H (1+L_rH)\|\mP_{T^*,h}(\cdot|s_h,a_h,\mu^{\pi^{k+1}}_{M^*,h}) - \mP_{T^{k+1},h}(\cdot|s_h,a_h,\mu^{\pi^{k+1}}_{M^{k+1},h})\|_\TV].
    \end{align*}
    Next, by applying Thm.~\ref{thm:model_diff_full} and Thm.~\ref{thm:accumu_model_diff}, w.p. $1-3\delta$, for any $\epsilon_0 > 0$, we have:
    \begin{align*}
        \sum_{k=1}^K \cE_\Opt(\pi^{k+1}) = O\Big((1+L_TH)(1+L_rH) \frac{(1+L_T)^H - 1}{L_T}\Big(\sqrt{K \dimE_\alpha(\cM,\epsilon_0) \log\frac{2|\cM|KH}{\delta}} + \alpha HK \epsilon_0\Big)\Big).
    \end{align*}
    Now take a look at Alg.~\ref{alg:Regret2PAC}, for each $n\in[N]$, by Markov inequality, with probability at least $\frac{2}{3}$:
    \begin{align}
        &\cE_\Opt(\pi^{k_n}) = J_{M^*}(\pi^*_\Opt) - J_{M^*}(\pi^{k_n}) \\
        \leq& 3\cdot \frac{1}{K} \cdot O\Big(H^2(1+L_TH)(1+L_rH) \frac{(1+L_T)^H - 1}{L_T}\Big(\sqrt{K \dimE_\alpha(\cM,\epsilon_0) \log\frac{2|\cM|KH}{\delta}} + \alpha HK \epsilon_0\Big)\Big).\\
        =& O\Big((1+L_TH)(1+L_rH) \frac{(1+L_T)^H - 1}{L_T}\Big(\sqrt{\frac{1}{K} \dimE_\alpha(\cM,\epsilon_0) \log\frac{2|\cM|KH}{\delta}} + \alpha H\epsilon_0\Big)\Big). \label{eq:Markov}
    \end{align}
    Since $\pi^{k_1},...,\pi^{k_N}$ are i.i.d. randomly selected, by choosing:
    \begin{align*}
        K = \tilde{O}\Big((1+L_TH)^2(1+L_rH)^2 \Big(\frac{(1+L_T)^H - 1}{L_T}\Big)^2 \frac{\dimE_\alpha(\cM,\epsilon_0)}{\epsilon^2}\Big)
    \end{align*}
    with $\epsilon_0 = O(\frac{L_T \epsilon}{\alpha H(1+L_TH)(1+L_rH)((1+L_T)^H - 1)})$, to make sure the RHS of Eq.~\eqref{eq:Markov} is less than $\frac{\epsilon}{2}$.
    Therefore, in Alg.~\ref{alg:Regret2PAC}, with probability $1-\delta$, we have
    \begin{align*}
        \exists n\in[N],\quad \cE_\Opt(\pi^{k_n}) \leq \frac{\epsilon}{2}.
    \end{align*}
    Then, by Hoeffding inequality, and note that the total return is upper bounded by 1, on good events of concentration, with probability $1-\delta$, we have:
    \begin{align*}
        \forall n\in[N],\quad |\hat{J}_{M^*}(\pi^{k_n}) - J_{M^*}(\pi^{k_n})| \leq \frac{\epsilon}{4}.
    \end{align*}
    which implies
    $$
    J_{M^*}(\hat\pi^{*}_{\Opt}) \geq \max_{n\in[N]} J_{M^*}(\pi^{k_n}) - \frac{\epsilon}{2} \geq J_{M^*}(\pi^*_\Opt) - \epsilon.
    $$
    Combining all the failure rate together, the above holds with probability at least $1-5\delta$.
    
    The analysis is similar with additional Assump.~\ref{assump:contraction}, where we have:
    \begin{align*}
        \sum_{k=1}^K \cE_\Opt(\pi^{k+1}) = O\Big(H^2(1+L_TH)(1+L_rH) (1+\frac{L_T}{1 - L_\Gamma})\Big(\sqrt{K \dimE(\cM,\epsilon_0) \log\frac{2|\cM|KH}{\delta}} + \alpha HK\epsilon_0\Big)\Big),
    \end{align*}
    and we should choose
    \begin{align*}
        K = \tilde{O}\Big(H^2(1+L_TH)^2(1+L_rH)^2 \Big(1 + \frac{L_T}{1 - L_\Gamma}\Big)^2 \frac{\dimE_\alpha(\cM,\epsilon_0)}{\epsilon^2}\Big),
    \end{align*} 
    with $\epsilon_0 = O(\frac{\epsilon}{\alpha H(1+L_TH)(1+L_rH)} (1+\frac{L_T}{1-L_\Gamma})^{-1})$.
\end{proof}
\subsection{Proofs for RL for Mean-Field Game}\label{appx:proofs_for_MFG}

\LemExploitDiff*
\begin{proof}
First of all,
\begin{align*}
    |\Delta_M(\tpi,\pi) - \Delta_\tM(\tpi,\pi)| = & |J_M(\tpi;\vecmu^\pi_M) - J_M(\pi;\vecmu^\pi_M) - J_\tM(\tpi;\vecmu^\pi_\tM) + J_\tM(\pi;\vecmu^\pi_\tM)| \\
    \leq & |J_M(\tpi;\vecmu^\pi_M) - J_\tM(\tpi;\vecmu^\pi_\tM)| + |J_M(\pi;\vecmu^\pi_M) - J_\tM(\pi;\vecmu^\pi_\tM)|.
\end{align*}
From Eq.~\eqref{eq:value_diff_2} of Lem.~\ref{lem:value_decomposition}, we have:
\begin{align*}
    &|J_M(\tpi;\vecmu^\pi_M) - J_\tM(\tpi;\vecmu^\pi_\tM)| \\
    \leq & \sum_{h=1}^H L_r\|\mu^\pi_{M,h} - \mu^\pi_{\tM,h}\|_\TV + \EE_{\tpi,M|\vecmu^\pi_M}[\sum_{h=1}^H \|\mP_{T,h}(\cdot|s_h,a_h,\mu^\pi_{M,h}) - \mP_{\tT,h}(\cdot|s_h,a_h,\mu^\pi_{\tM,h})\|_\TV].
\end{align*}
By choosing $\tpi = \pi$, the above implies
\begin{align*}
    &|J_M(\pi;\vecmu^\pi_M) - J_\tM(\pi;\vecmu^\pi_\tM)| \\
    \leq & \sum_{h=1}^H L_r\|\mu^\pi_{M,h} - \mu^\pi_{\tM,h}\|_\TV + \EE_{\pi,M}[\sum_{h=1}^H \|\mP_{T,h}(\cdot|s_h,a_h,\mu^\pi_{M,h}) - \mP_{\tT,h}(\cdot|s_h,a_h,\mu^\pi_{\tM,h})\|_\TV\Big].
\end{align*}
Therefore,
\begin{align*}
    |\Delta_M(\tilde\pi,\pi) - \Delta_\tM(\tilde\pi,\pi)| \leq & 2\sum_{h=1}^H L_r \|\mu^\pi_{M,h} - \mu^\pi_{\tM,h}\|_\TV \\
    & + \EE_{\tpi,M|\vecmu^\pi_M}[\sum_{h=1}^H \|\mP_{T,h}(\cdot|s_h,a_h,\mu^\pi_{M,h}) - \mP_{\tT,h}(\cdot|s_h,a_h,\mu^\pi_{\tM,h})\|_\TV]\\
    & + \EE_{\pi,M}[\sum_{h=1}^H \|\mP_{T,h}(\cdot|s_h,a_h,\mu^\pi_{M,h}) - \mP_{\tT,h}(\cdot|s_h,a_h,\mu^\pi_{\tM,h})\|_\TV].
\end{align*}
where we have:
\begin{align*}
    \sum_{h=1}^H \|\mu^\pi_{M,h} - \mu^\pi_{\tM,h}\|_\TV \leq H\EE_{\pi,M}[\sum_{h=1}^H \|\mP_{T,h}(\cdot|s_h,a_h,\mu^\pi_{M,h}) - \mP_{\tT,h}(\cdot|s_h,a_h,\mu^\pi_{\tM,h})\|_\TV].
\end{align*}
As a result of Corollary.~\ref{coro:accum_density_err}, and we finish the proof.
\end{proof}

\begin{theorem}[Result for MFG; Full Version of Thm.~\ref{thm:short_main_results_MFC_MFG} and Thm.~\ref{thm:short_main_results_MFC_MFG_contractivity}]\label{thm:MFG_main_full}
    Under Assump.~\ref{assump:realizability} and \ref{assump:lipschitz}, by running Alg.~\ref{alg:MLE_MB_Alg} with the MFG branch, after consuming $2HK$ trajectories, where $K$ is set to
    \begin{align*}
        K = \tilde{O}\Big(H^2(1 + L_TH)^2(1+L_rH)^2 \Big(\frac{(1+L_T)^H - 1}{L_T}\Big)^2 \frac{\dimE_\alpha(\cM,\epsilon_0)}{\epsilon^2}\Big),
    \end{align*}
    where $\epsilon_0 = O(\frac{L_T \epsilon}{\alpha H(1 + L_TH)(1+L_rH)((1+L_T)^H - 1)})$;
    or set to the following with additional Assump.~\ref{assump:contraction}:
    \begin{align*}
        K = \tilde{O}\Big(H^2(1 + L_TH)^2(1+L_rH)^2 \Big(1 + \frac{L_T}{1 - L_\Gamma}\Big)^2 \frac{\dimE_\alpha(\cM,\epsilon_0)}{\epsilon^2}\Big),
    \end{align*} 
    where $\epsilon_0 = O(\frac{\epsilon}{\alpha H(1 + L_TH)(1+L_rH)} (1+\frac{L_T}{1-L_\Gamma})^{-1})$, with probability at least $1-5\delta$, we have $\cE_\NE(\hat\pi^*_\NE) \leq \epsilon$.
\end{theorem}
\begin{proof}
In the following, we use $\cE_\NE^M(\pi):=\max_\tpi\Delta_M(\tpi, \pi)$ to denote the exploitability in model $M$.
Recall $M^{k+1}$ denotes the model such that $\pi^{k+1}$ is one of its equilibrium policies satisfying $\cE_\NE^{M^{k+1}}(\pi^{k+1}) = 0$. On the event in Thm.~\ref{thm:MLE_Gaurantee}, $\forall k\in[K]$, we have $M^* \in \hat\cM^k$, which implies
\begin{align*}
    \cE_\NE(\pi^{k+1}) \leq& \cE_\NE^{\tM^{k+1}}(\pi^{k+1})\\
    = & \Delta_{\tM^{k+1}}(\tpi^{k+1},\pi^{k+1})\\
    \leq & \Delta_{\tM^{k+1}}(\tpi^{k+1},\pi^{k+1}) - \Delta_{M^{k+1}}(\tpi^{k+1},\pi^{k+1}) \tag{$\pi^{k+1}$ is an equilibrium policy of $M^{k+1}$ so $\Delta_{M^{k+1}}(\tpi^{k+1},\pi^{k+1}) \leq 0$}\\
    \leq & |\Delta_{\tM^{k+1}}(\tpi^{k+1},\pi^{k+1}) - \Delta_{M^*}(\tpi^{k+1},\pi^{k+1})| + |\Delta_{M^*}(\tpi^{k+1},\pi^{k+1}) - \Delta_{M^{k+1}}(\tpi^{k+1},\pi^{k+1})|.
\end{align*}
By applying Lem.~\ref{lem:exploitability_diff}, Coro.~\ref{coro:accum_density_err}, and Thm.~\ref{thm:accumu_model_diff}, under Assump.~\ref{assump:lipschitz}, we have:
\begin{align*}
    \sum_{k=1}^K \cE_\NE(\pi^{k+1}) \leq& \sum_{k=1}^K \cE_\NE^{\tM^{k+1}}(\pi^{k+1}) \\
    =& O\Big((1 + L_TH)(1+L_rH) \frac{(1+L_T)^H - 1}{L_T}\Big(\sqrt{K \dimE_\alpha(\cM,\epsilon_0) \log\frac{2|\cM|KH}{\delta}} + \alpha KH\epsilon_0\Big)\Big).
\end{align*}
For the choice of $\hat\pi^*_\NE$, since
\begin{align*}
\cE_\NE(\hat\pi^*_\NE) \leq \min_{k\in[K]} \cE_\NE^{\tM^{k+1}}(\pi^{k+1}) \leq \frac{1}{K} \sum_{k=1}^K \cE_\NE^{\tM^{k+1}}(\pi^{k+1}),
\end{align*}
$\cE_\NE(\hat\pi^*_\NE) \leq \epsilon$ can be ensured by:
\begin{align*}
    K = \tilde{O}\Big(H^2(1 + L_TH)^2(1+L_rH)^2 \Big(\frac{(1+L_T)^H - 1}{L_T}\Big)^2 \frac{\dimE_\alpha(\cM,\epsilon_0)}{\epsilon^2}\Big),
\end{align*}
where $\epsilon_0 = O(\frac{L_T \epsilon}{\alpha H(1 + L_TH)(1+L_rH)((1+L_T)^H - 1)})$.

Given additional Assump.~\ref{assump:contraction}, we have:
\begin{align*}
    \sum_{k=1}^K \cE_\NE(\pi^{k+1}) \leq& \sum_{k=1}^K \cE_\NE^{\tM^{k+1}}(\pi^{k+1}) \\
    =& O\Big(H^2(1+L_TH)(1+L_rH) (1+\frac{1}{1 - L_\Gamma})\Big(\sqrt{K \dimE_\alpha(\cM,\epsilon_0) \log\frac{2|\cM|KH}{\delta}} + \alpha KH\epsilon_0\Big)\Big).
\end{align*}
$\cE_\NE(\hat\pi^*_\NE) \leq \epsilon$ can be ensured by
\begin{align*}
    K = \tilde{O}\Big(H^2(1 + L_TH)^2(1+L_rH)^2 \Big(1 + \frac{L_T}{1 - L_\Gamma}\Big)^2 \frac{\dimE_\alpha(\cM,\epsilon_0)}{\epsilon^2}\Big),
\end{align*} 
where $\epsilon_0 = O(\frac{\epsilon}{\alpha H(1 + L_TH)(1+L_rH)} (1+\frac{L_T}{1-L_\Gamma})^{-1})$.

We finish the proof by noting that the total failure rate would be $1-3\delta$, and the total sample complexity would be $2HK$.
\end{proof}

\section{Questions Concerning Existence and Imposed Conditions}\label{appx:existence_Eqb}

In this section, we analyze the existence of MFG-NE in the game described and discuss when the presented conditions might be satisfied.
For clarity in notation, we fix the model $M  = (\{\mathbb{P}_{T,h}\}_{h=1}^H, \{\mathbb{P}_{r,h}\}_{h=1}^H)$ and the initial distribution $\mu_1$, and also for simplicity denote the deterministic expected rewards
\begin{align*}
    r_h(s,a,\mu) := \mathbb{E}_{r \sim \mathbb{P}_{r,h}(\cdot|s,a,\mu)} \left[ r \right],
\end{align*}
since the probabilistic distribution of rewards will not be significant for existence results.
In the presented MFG-NE problem, the goal is to find a sequence of policies $\pi := \{\pi_h\}_{h=1}^{H}$ and a sequence of population distributions $\vecmu=\{\mu_h\}_{h=1}^{H}$ such that
\begin{align*}
    &\textbf{Consistency: } \mu_{h+1} = \Gamma_{pop,h} (\mu_h, \pi_h), \forall h = 1,\ldots,H-1, \\
    &\textbf{Optimality: } J_M(\pi, \vecmu) = \max_{\pi'} J_M(\pi', \vecmu) 
\end{align*}
where $\mu_1$ is fixed and for any $\vecmu=\{\mu_h\}_{h=1}^{H}$, $\pi := \{\pi_h\}_{h=1}^{H}$, with $\mu_h\in\Delta(\cS_h)$ and $\pi_h\in \Pi_h:=\{ \pi_h: \cS_h \rightarrow \Delta(\cA_h)\}$. We define:
\begin{align*}
\Gamma_{pop,h} (\mu_h,\pi_h) &:= \SumInt_{s_h\in\cS_h}\SumInt_{a_h\in\cA_h}\mu_h(s_h) \pi_h(a_h|s_h) \mathbb{P}_{T,h}(\cdot|s_h,a_h,\mu_h), \\
    J_M(\pi, \vecmu) &:= \mathbb{E}\left[ \sum_{h=1}^{H} r_h(s_h, a_h, \mu_h) \middle|  \substack{s_1\sim \mu_1,~a_h \sim \pi_h \\ s_{h+1}\sim \mP_{T,h}(\cdot|s_h,a_h,\mu_h)},~~\forall h \geq 1\right].
\end{align*}
As a general strategy, we formulate in this section the two MFG-NE conditions above as fixed point problems.
Throughout this section, we will assume the following:
\begin{assumption}[Continuous rewards and dynamics]\label{existence:assumption_cont}
For each $h\in[H]$, $(s_h,a_h, s_{h+1}) \in \cS_h \times \cA_h \times \cS_{h+1}$, the mappings
\begin{align*}
    \mu \rightarrow r_h(s_h,a_h,\mu);\quad \mu \rightarrow \mathbb{P}_{T,h}(s_{h+1}|s_h,a_h,\mu)
\end{align*}
are continuous, where $\Delta(\cS)$ is equipped with the total variation distance $\TV$.
\end{assumption}

\subsection{Strict MFG-NE as a Fixed Point}
We first introduce a stronger notion of NE, which we call the ``Strict NE''.
\begin{definition}[Strict MFG-NE]\label{def:strict_NE}
    We call the policy $\pi^*$ a strict NE of if and only if the following holds for each $h,s$,
    \begin{align*}
        \pi^*_h(\cdot|s) &= \argmax_{u \in \Delta_\mathcal{A}} Q_h^{\pi^*}(s,\cdot, \vecmu^*)^\top u.
    \end{align*}
\end{definition}
Note that a strict NE is always a NE. In the following, we only focus on the existence of strict NE.

We use the standard definition of Q-value functions on finite horizon MF-MDPs, for any $\bar{h}, s, a, \pi, \vecmu$ given by
\begin{align}
    Q_{\bar{h}}^{\pi}(s_{\bar{h}},a_{\bar{h}}, \vecmu) := \mathbb{E}\left[ \sum_{h=\bar{h}}^{H} r_{h}(s_{h}, a_{h}, \mu_{h}) \middle| a_{h} \sim \pi_{h}(s_{h}), s_{h+1}\sim \mathbb{P}_{T,h}(\cdot|s_{h}, a_{h}, \mu_{h}),~\forall~h > \bar{h}\right].\label{eq:cond_Q_function}
\end{align}
Observe that the set of policies and $\Delta(\cS)$ are both convex and closed sets (in fact, polytopes), given by $\{\Delta(\cS_h)\}_{h\in[H]}, \{\Pi_h\}_{h\in[H]}$.
We equip these sets with the metrics
\begin{align*}
    \forall \pi,\pi'\in\{\Pi_h\}_{h\in[H]},\quad d_1(\pi, \pi') &:= \sup_{h} \| \pi_h - \pi'_h\|_2 \\
    \forall \vecmu,\vecmu'\in\{\Delta(\cS_h)\}_{h\in[H]},\quad d_2(\vecmu, \vecmu') &:= \sup_{h} \| \mu_h - \mu'_h\|_2.
\end{align*}
We also define the operators $\Gamma_{pop}:\{\Pi_h\}_{h\in[H]} \rightarrow \{\Delta(\cS_h)\}_{h\in[H]}$ and $\Gamma_{pp}: \{\Pi_h\}_{h\in[H]} \times \{\Delta(\cS_h)\}_{h\in[H]} \rightarrow \{\Pi_h\}_{h\in[H]}$ as
\begin{align*}
    \Gamma_{pop}(\pi) &:= \{\mu_1\} \cup \{\mu_{h+1} := \underbrace{(\Gamma_{pop}(\pi_{h}, \ldots \Gamma_{pop,2}(\pi_2,\Gamma_{pop,1}(\pi_1,\mu_1)))}_{\text{from 1 to $h$}} \}_{h=1}^{H-1} , \\
    \Gamma_{pp}(\pi, \vecmu) &:= \{ \pi'_h(\cdot|s_h) := \argmax_{u \in \Delta_\mathcal{A}} Q_h^\pi(s_h,\cdot, \vecmu)^\top u - \| \pi_h(\cdot|s_h) - u\|_2^2\}_{h=1}^{H},
\end{align*}
where $Q_h^\pi$ is the Q-value function defined in Eq.~\eqref{eq:cond_Q_function}.
The motivation for these operators is given by the following lemma:

\begin{lemma}[Strict MFG-NE as fixed point]\label{lemma:existencefixedpointreduction}
The tuple $\pi^*, \vecmu^*$ is a strict MFG-NE if and only if the following conditions hold:
\begin{enumerate}
    \item $\pi^* = \Gamma_{pp}( \pi^*, \Gamma_{pop}(\pi^*))$, that is, $\pi^*$ is a fixed point of $\Gamma_{SNE}(\cdot) := \Gamma_{pp}( \cdot, \Gamma_{pop}(\cdot))$.
    \item $\vecmu^* = \Gamma_{pop}(\pi^*)$.
\end{enumerate}
\end{lemma}

\begin{proof}
First, assume $(\pi^*, \vecmu^*)$ is a strict MFG-NE, i.e., it satisfies the consistency and optimality conditions.
By consistency, we have $\Gamma_{pop}(\pi^*) = \vecmu^*$, and since this implies $\Gamma_{pp}(\pi^*, \vecmu^*) = \pi^*$, the optimality condition implies for each $h,s$,
\begin{align*}
    \pi^*_h(\cdot|s) &= \argmax_{u \in \Delta_\mathcal{A}} Q_h^{\pi^*}(s,\cdot, \vecmu^*)^\top u.
\end{align*}
which implies that
\begin{align*}
     \pi^*_h(\cdot|s) &= \argmax_{u \in \Delta_\mathcal{A}} Q_h^{\pi^*}(s,\cdot, \vecmu^*)^\top u - \| \pi_h^*(\cdot|s) - u\|_2^2,
\end{align*}
that is, $\Gamma_{SNE}(\pi^*) = \pi^*$.

Conversely, assume $\pi^* = \Gamma_{SNE}(\pi^*)$, that is, $\pi^*$ is a fixed point of the operator $\Gamma_{SNE}$.
We claim that $(\pi^*, \vecmu^* = \Gamma_{pop}(\pi^*))$ is a MFG-NE.
For this pair, the consistency condition is satisfied by definition, and the fixed point condition reduces to $\Gamma_{pp}(\pi^*, \vecmu^*) = \pi^*$.
Writing out the definition of the $\Gamma_{pp}$ operator, we obtain for each $h$ and $s_h$,
\begin{align*}
    \pi^*_h(\cdot|s) &= \argmax_{u \in \Delta_\mathcal{A}} Q_h^{\pi^*}(s,\cdot, \vecmu^*)^\top u - \| \pi_h^*(\cdot|s) - u\|_2^2, \\
    \pi^*_h(\cdot|s) &= \argmax_{u \in \Delta_\mathcal{A}} Q_h^{\pi^*}(s,\cdot, \vecmu^*)^\top u,
\end{align*}
by the first-order optimality conditions of the term $Q_h^{\pi^*}(s,\cdot, \vecmu^*)^\top u - \| \pi_h(\cdot|s) - u\|_2^2$.
We finish the proof.
\end{proof}

In the lemma above, the second condition is trivial to satisfy/compute once $\pi^*$ is known, hence the primary challenge will be in proving that the map $\Gamma_{SNE}$ admits a fixed point.

\subsection{Existence of MFG-NE}

We use the Brower fixed point method to prove the existence of a MFG-NE, and Assump.~\ref{existence:assumption_cont} is sufficient.
The strategy will be to show that $\Gamma_{SNE}$ is a continuous function on the compact and convex policy/population distribution space.

We will prove several continuity results, in order to be able to apply Brouwer's fixed point theorem.

\begin{lemma}[Continuity of $Q_h^\pi$]\label{existence:qcont}
For any $s,a, h$, the map
\begin{align*}
    \pi, \vecmu \rightarrow Q_h^\pi(s,a, \vecmu) \in \mathbb{R}
\end{align*}
is continuous.
\end{lemma}
\begin{proof}
The proof follows from the fact that $Q_h^\pi$ is a function of sum and multiplications of continuous functions of the policies and population distributions $\{\pi_h\}_{h\in[H]}, \{\mu_h\}_{h\in[H]}$.
The compositions, additions and multiplications of continuous functions are continuous.
\end{proof}

For the next continuity result, we will need the following well-known Fenchel conjugate definition and duality.
\begin{definition}[Fenchel conjugate]
Assume that $f: \mathbb{R}^d \rightarrow \mathbb{R}\cup \{\infty\}$ is a convex function, with domain $\cX \subset \mathbb{R}^d$.
The Fenchel conjugate $f^*: \mathbb{R}^d \rightarrow \mathbb{R}\cup \{\infty\}$ is defined as 
\begin{align*}
    f^*(y) = \sup_{x \in \cX} \langle x, y\rangle - f(x).
\end{align*}
\end{definition}
For further details regarding the Fenchel conjugate, see \citep{nesterov2018lectures}.
The Fenchel conjugate is useful due to the following well-known duality result.
\begin{lemma}\label{existence:lemma:fenchel_duality}
Assume that $f: \mathbb{R}^d \rightarrow \mathbb{R}\cup \{\infty\}$ is differentiable and $\tau$-weakly convex and has domain $\cX \subset \mathbb{R}^d$.
Then,
\begin{enumerate}
    \item $f^*$ is differentiable on $\mathbb{R}^d$,
    \item $\nabla f^*(y) = \argmax_{x\in \cX} \langle x,y\rangle - f(x)$,
    \item $f^*$ is $\frac{1}{\tau}$-smooth with respect to $\|\cdot\|_2$, i.e., $\|\nabla f^*(y) - \nabla f^*(y')\| \leq \frac{1}{\tau} \| y - y'\|_2, \forall y,y' \in \mathbb{R}^d$.
\end{enumerate}
\end{lemma}
\begin{proof}
See Lemma 15 of \citep{shalev2007online} or Lemma 6.1.2 of \citep{nesterov2018lectures}.
\end{proof}

Finally, we will also need the non-expansiveness of the proximal point operator, presented below.
\begin{lemma}[Proximal operator is non-expansive \citep{parikh2014proximal}]\label{existence:pp_non_expansive}
Let $\cX\subset \mathbb{R}^d$ be a compact convex set, and $f:\cX\rightarrow\mR$ be a convex function.
The proximal map $\operatorname{prox}_f: \cX \rightarrow \cX$ defined by
\begin{align*}
    \operatorname{prox}_f (x) := \argmin_{y\in \cX} f(y) + \| x - y\|_2^2
\end{align*}
is non-expansive (hence continuous).
\end{lemma}
With the presented tools, we can prove the following statement.
\begin{lemma}[Continuity of $\Gamma_{pop}, \Gamma_{pp}$]
With the metrics $d_1, d_2$, the operators $\Gamma_{pop}, \Gamma_{pp}$ are Lipschitz continuous mappings.
\end{lemma}
\begin{proof}
The continuity of $\Gamma_{pop}$ w.r.t. $\pi$ is straightforward by definition, as multiplications and additions of continuous functions are continuous.

For the continuity of $\Gamma_{pp}$, we can either explicitly write the solution of the $\argmax$ problem in terms of an affine function and a projection of terms $Q^{\pi_h}_h, \pi_h$, or more generally use Fenchel duality combined with the non-expansiveness of the proximal point operator.
By Lemma~\ref{existence:pp_non_expansive}, the map
\begin{align*}
    u \rightarrow \argmax_{u'\in\Delta_\mathcal{A}} q^\top u' - \|u - u' \|_2^2 = - \operatorname{prox}_{-q^\top(\cdot)} (u)
\end{align*}
is a continuous map for any $q \in \mathbb{R}^{|\mathcal{A}|}$.
Similarly, by Lemma~\ref{existence:lemma:fenchel_duality}, the map
\begin{align*}
    q \rightarrow \argmax_{u'\in\Delta_\mathcal{A}} q^\top u' - \| u' - u \|_2^2
\end{align*}
is differentiable hence continuous for any $u\in\Delta_{\mathcal{A}}$, as the map $\|u - \cdot\|_2^2$ is weakly convex.
By the continuity of $Q_h^\pi$ (see Lemma~\ref{existence:qcont}), we can conclude that $\Gamma_{pp}$ is also a continuous map, as it is the composition of continuous functions.
\end{proof}

With this continuity characterization, we invoke Brouwer's fixed point theorem to prove existence.
\begin{proposition}[Existence of MFG-NE; Formal Version of Prop.~\ref{prop:exist_MFG_NE_informal}]\label{prop:exist_MFG_NE_formal}
    Under Assump.~\ref{existence:assumption_cont} (which is implied by Assump.~\ref{assump:lipschitz}), the map $\Gamma_{SNE}$ has a fixed point in the set $\{\Pi_h\}_{h\in[H]}$, that is, there exists a $\pi^*$ such that $\Gamma_{SNE}(\pi^*) = \pi^*$, and the tuple $(\pi^*, \Gamma_{pop}(\pi^*))$ is a strict MFG-NE, which implies the existence of NE.
\end{proposition}
\begin{proof}
With the continuity of $\Gamma_{pop}, \Gamma_{pp}$, the know that the composition $\Gamma_{SNE}$ is continuous.
It maps the closed, convex polytope $\{\Pi_h\}_{h\in[H]}$ to a subset of itself, hence by Brouwers fixed point theorem it must admit a fixed point.
By Lemma~\ref{lemma:existencefixedpointreduction}, this fixed point must constitute a strict MFG-NE.

Comparing with Eq.~\eqref{eq:objective_MFG} and Def.~\ref{def:strict_NE}, we know the existence of strict NE implies the existence of NE.
\end{proof}

\end{document}